\documentclass[letterpaper]{article} % DO NOT CHANGE THIS
\usepackage{aaai24} 
\usepackage{times}  % DO NOT CHANGE THIS
\usepackage{helvet}  % DO NOT CHANGE THIS
\usepackage{courier}  % DO NOT CHANGE THIS
\usepackage[hyphens]{url}  % DO NOT CHANGE THIS
\usepackage{graphicx} % DO NOT CHANGE THIS
\usepackage{natbib}  % DO NOT CHANGE THIS AND DO NOT ADD ANY OPTIONS TO IT
\usepackage{caption} % DO NOT CHANGE THIS AND DO NOT ADD ANY OPTIONS TO IT
\usepackage{algorithm}
\usepackage{algorithmic}

\usepackage{newfloat}
\usepackage{listings}
\DeclareCaptionStyle{ruled}{labelfont=normalfont,labelsep=colon,strut=off} % DO NOT CHANGE THIS
\floatstyle{ruled}
\newfloat{listing}{tb}{lst}{}
\floatname{listing}{Listing}
\usepackage{amsmath,amsfonts,bm,amssymb}

\def\1{\bm{1}}

\newcommand{\gaussconst}{\ensuremath{C_{\sigma}}}

\DeclareMathAlphabet{\mathsfit}{\encodingdefault}{\sfdefault}{m}{sl}
\SetMathAlphabet{\mathsfit}{bold}{\encodingdefault}{\sfdefault}{bx}{n}

\newcommand{\R}{\mathbb{R}}

\usepackage{amsthm}

\usepackage{mathrsfs}
\DeclareRobustCommand{\bigO}{%
  \text{\usefont{OMS}{cmsy}{m}{n}O}%
}
\allowdisplaybreaks

\newcommand{\graph}{\ensuremath{\mathcal{G}}}

\newcommand{\nodes}{\ensuremath{\mathcal{V}}}
\newcommand{\edges}{\ensuremath{\mathcal{E}}}

\newcommand{\landmarks}{\ensuremath{\mathfrak{L}}}
\newcommand{\witness}{\ensuremath{\mathrm{Wit}}}

\makeatletter
\newtheorem*{rep@theorem}{\rep@title}
\newcommand{\newreptheorem}[2]{%
	\newenvironment{rep#1}[1]{%
		\def\rep@title{\textbf{#2} \ref{##1}}%
		\begin{rep@theorem}}%
		{\end{rep@theorem}}}
\makeatother
\newreptheorem{theorem}{Theorem}
\newreptheorem{lemma}{Lemma}
\newreptheorem{proposition}{Proposition}
\newreptheorem{assumption}{Assumption}
\newreptheorem{corollary}{Corollary}
\newreptheorem{definition}{Definition}

\usepackage{longtable,tcolorbox}
\usepackage{amsmath}
\usepackage{amssymb}
\usepackage{mathtools}
\usepackage{amsthm}
\theoremstyle{plain}
\newtheorem{theorem}{Theorem}[section]
\newtheorem{proposition}[theorem]{Proposition}

\theoremstyle{definition}
\newtheorem{definition}[theorem]{Definition}

\theoremstyle{remark}
\newtheorem{remark}[theorem]{Remark}

\usepackage{microtype}
\usepackage{booktabs} % for professional tables
\usepackage{rotating}
\usepackage{url}
\usepackage{array}
\usepackage{multirow}

\usepackage{subfigure}
\nocopyright %-- Your paper will not be published if you use this command
\title{When Witnesses Defend:\\ A Witness Graph Topological Layer for Adversarial Graph Learning}
\author {
    Naheed Anjum Arafat\textsuperscript{\rm 1},
    Debabrota Basu\textsuperscript{\rm 2},
    Yulia Gel\textsuperscript{\rm 3},
    Yuzhou Chen\textsuperscript{\rm 4}
}
\affiliations {
    \textsuperscript{\rm 1}Nanyang Technological University, Singapore\\
    \textsuperscript{\rm 2}Univ. Lille, Inria, CNRS, Centrale Lille, UMR 9189 – CRIStAL, F-59000 Lille, France\\
    \textsuperscript{\rm 3}Virginia Tech, USA\\
    \textsuperscript{\rm 4}University of California, Riverside, USA\\
    naheed\_anjum@u.nus.edu, debabrota.basu@inria.fr, 
    ygl@vt.edu,
    yuzhouc@ucr.edu
}

\begin{document}

\maketitle

\begin{abstract}
Capitalizing on the intuitive premise that shape characteristics are more robust to perturbations,
we bridge adversarial graph learning with the emerging tools from computational topology, namely, persistent homology representations of graphs.
We introduce the concept of witness complex to adversarial analysis on graphs, which allows us to focus only on the salient shape characteristics of graphs, yielded by the subset of the most essential nodes (i.e., landmarks), with minimal loss of topological information on the whole graph. The remaining nodes are then used as witnesses, governing which higher-order graph substructures are incorporated into the learning process. Armed with the witness mechanism, we design \emph{Witness Graph Topological Layer (WGTL)}, which systematically integrates both local and global topological
graph feature representations, the impact of which is, in turn, automatically controlled by the robust regularized topological loss. Given the attacker's budget, we derive the important stability guarantees of both local and global topology encodings and the associated robust topological loss. We illustrate the versatility and efficiency of WGTL by its integration with five GNNs and three existing non-topological defense mechanisms. Our extensive
experiments across six datasets demonstrate that WGTL boosts the robustness of GNNs across a range of perturbations and against a range of adversarial attacks.
% leading to relative gains of up to 18\%. 
Our datasets and source codes are available at \url{https://github.com/toggled/WGTL}.
\end{abstract}

% Uncomment the following to link to your code, datasets, an extended version or similar.
%
% \begin{links}
%     \link{Code}{https://aaai.org/example/code}
%     \link{Datasets}{https://aaai.org/example/datasets}
%     \link{Extended version}{https://aaai.org/example/extended-version}
% \end{links}

\section{Introduction}
%Graphs are ubiquitous data structures with applications in numerous domains, such as chemistry, finance, biology and social media. With their prevalence, it is important to learn effective representations of graphs and then apply them to solve downstream learning tasks, including node classification, link prediction, representation learning and recommender systems~\citep{zhou2020graph}. In recent years, Graph Neural Networks (GNNs) have witnessed great success for representation learning of graphs~\citep{wu2020comprehensive}. GNNs follow a message-passing scheme, where the node embedding is obtained by aggregating and transforming the embeddings of its neighbours~\citep{kipf}. 

% Graphs are ubiquitous data structures with applications in numerous knowledge domains: from structural representation of molecules in chemistry and material science to cryptocurrency transaction networks in finance.
% %to biotic interactions in ecology. 
% With their prevalence, it is important to learn effective graph representations and then apply them to solve downstream learning tasks. %such as node classification, link prediction, and anomaly detection. 
% In present, the most widely adopted machinery for graph learning tasks is arguably Graph Neural Networks (GNNs)~\cite{gnnsurvey2022}. 
% However, similar to the deep neural networks (DNN), GNNs exhibit vulnerability to adversarial attacks

Recent studies have shown that Graph neural networks (GNNs) are vulnerable to adversarial attacks. Small, often unnoticeable perturbations to the input graph might result in substantial degradation of GNN's performance in downstream tasks~\cite{survey2021}.
In turn, compared to non-graph data, adversarial analysis of graphs still remains largely underexplored~\cite{sun2022adversarial}. 
%Hence, developing robust GNN models that can resist a wide spectrum of adversarial attacks is of significant practical importance.
Hence, systematic assessment of adversarial contagions and consequently development of robust GNN models able to withstand a wide spectrum of malicious attacks are of significant practical importance. 

Presently, the three main strategies to defend GNNs against adversarial attacks are graph purification, adversarial training, and adversarial defense based neural architectures~\citep{feng2019graph, gunnemann2022graph,metapgd}.
%Among such recent approaches within the class of graph purification tools are Pro-GNN~\citep{prognn} and SVD-GCN~\citep{svdgcn}, while adversarial training approaches include, for example,~\cite {flag, pgdattack}. Finally, the family of the enhanced adversarial defense architectures includes, for instance, VAE~\citep{vgae}, Bayesian uncertainty quantification~\citep{feng2021uag}, and Attention~\citep{tang2020}. 
% These existing methods largely rely on the information at a node level while ignoring the higher-order, multi-scale properties of the graph structure, which are often the key behind the success of the learning task~\citep{benson2018simplicial,torres2021and}. 
These existing methods largely rely on pairwise relationships in the graph at a node level while ignoring higher-order graph (sub)structures, their multi-scale properties, and interrelationships, which are instrumental for the downstream learning task~\citep{benson2018simplicial,torres2021and}. Relying on pairwise relationships also results in the removal of a considerable amount of the edges that are actually clean edges, which decreases the gain in robustness~\cite{in2024self}. 

% In contrast, WGTL primarily focuses on learning the key higher-order graph interactions at both local and global levels and then adaptively assessing their potential defense role via topological regularizer. Furthermore, the local and global topological encodings remain robust despite the false positive edges; as a result, WGTL alleviates the problems associated with false positive edges, enhancing the overall resilience.

%Defense approaches such as adversarially robust training~\citep{xu2019_advtr} and graph purification~\citep{prognn} have been proposed to defend against adversarial attacks. However, none of these approaches explored how to make existing GNN models robust by encoding adversarially robust features. This paper proposes adversarially robust features and components that can be incorporated into existing GNNs to enhance their robustness against adversarial attacks. 

In turn, in the last few years, we have observed a spike of interest in the synergy of graph learning and Persistent Homology (PH) representations of graphs~\citep{zhao2019learning,carriere2020perslay,horn2021topological,yan2022neural,chen2022structure,topopooling}. 
PH representations enable us to
glean intrinsic information about the inherent object shape. By shape here, we broadly understand properties that are invariant under continuous transformations such as twisting, bending, and stretching.
This phenomenon can be
explained by the important higher-order information, which PH-based shape descriptors deliver about the underlying graph-structured data.
This leads to an enhanced GNN performance in a variety of downstream tasks, such
as link prediction, and node and graph classification~\citep{hofer2020graph, carriere2020perslay,yan2021link, horn2021topological, chen2022topoattn}. 
Furthermore, in view of the invariance with respect to continuous transformations, intuitively we can expect that shape characteristics are to yield higher robustness to random perturbations and adversarial attacks. While this intuitive premise of robustness and its relationship with DNN architectures has been confirmed by some recent studies~\citep{chen2021topological, gebhart2019characterizing,goibert2022adversarial}, to the best of our knowledge, there are no attempts to incorporate PH-based graph representations for adversarial defense.
% yet exists no topological adversarial defense for GNNs.

% 
% 
% 
%The proposed \textit{persistent homological features} originates from the literature of Persistent homology~\citep{edelsbrunner2000topological}. It is well-known that such features encapsulate higher-order, multi-scale topological information of graphs~\citep{edelsbrunner2000topological} and are robust against small perturbations to the adjacency matrix~\citep{stability2005}.  In this paper, we exploit this \emph{robustness to noise} property of persistent homological features to demonstrate that the encoding topological features as priors into GNNs enhance their robustness against adversarial attacks.

%There are two types of persistent homological features: exact features as computed by Cech and Vietoris Rips and approximations such as Witness features~\citep{arafat2019topological}. Since the exact computation of these features is computationally expensive, we propose to use Witness features in this paper to reduce pre-processing time. 

In this paper, we bridge this gap by merging adversarial graph learning with PH representations of graph-structured data.
Our key idea is to leverage the concept of witness complex for graph learning. This allows us firstly, to enhance the computational efficiency of the proposed 
topological defense, which is one of the primary bottlenecks for the wider adoption of topological methods, and lastly, to reduce the impact of less important or noisy graph information.
In particular, the goal of the witness complex is to accurately 
estimate the intrinsic shape properties of the graph using not all available graph information, but {\it only} a subset of the most representative nodes, called {\it landmarks}. The remaining nodes are then used as {\it witnesses}, governing which higher-order graph substructures 
shall be incorporated into the process of extracting shape characteristics and the associated graph learning task. 
% Intuitively, the idea can be compared with focusing only on the shape of the object skeleton, which is invariant under deformations. 
% That is, to distinguish between two mammals, in most cases we do not need information on all bones such as legs, feet, and shoulder but only on the axial skeleton. 
This mechanism naturally results in two main benefits. 
First, it allows us to drastically reduce the computational costs. Second, it allows us to extract salient shape characteristics (i.e., skeleton shape). Our topological defense takes the form of the \emph{Witness Graph Topological Layer (WGTL)} with three novel components: \textit{local and global witness complex-based topological
encoding}, \textit{topology prior aggregation}, and \textit{robustness-inducing topological loss}. 

The \textit{local witness complex-based features} encapsulate graph topology within the local node neighborhoods, while the \textit{global witness complex-based features} describe global graph topology. 
% \YGL{Naheed, we need some words on the role of Topology prior aggregation}
Using only local topology prior to the loss function might be vulnerable to local attacks, while only global topology prior might be more susceptible to global attacks. To defend against both types of attacks, both local and global topology prior needs to be combined, thus motivating the design of the topology prior aggregator.
Inspired by studies such as~\citet{hu2019topology,  carriere2021optimizing}, we \textit{use the robust topological loss as a regularizer to a supervised loss for adversarially robust node representation learning}. This allows for control over which shape features are to be included in the defense mechanism. 
% 
% Furthermore, given an attacker's budget, we theoretically derive the stability guarantees of both local and global topology encodings, and the associated topological loss. 
% Figure~\ref{fig:schematic} shows the schematic of the proposed components.
% The proposed WGTL is versatile as the proposed shape features can be readily integrated with any GNN architecture.
% Our extensive numerical experiments in conjunction with node classification tasks demonstrate that WGTL enhances performance of GNNs on clean graphs, as well as substantially improves their robustness again a broad range of adversarial attacks. Furthermore, we also demonstrate that WGTL can be incorporated to boost the robustness capabilities of existing graph defense mechanisms such as Pro-GNN~\citep{prognn}. 
% 

\noindent\textbf{Our Contributions.} Our contributions are summarized as follows:
\begin{itemize}
% [leftmargin=*, noitemsep,topsep=4pt]

\item We propose the first approach which systematically bridges adversarial graph learning with persistent homology representations of graphs.

\item We introduce a novel topological adversarial defense for graph learning, i.e. the \emph{Witness Graph Topological Layer (WGTL)}, based on the notion of the witness complex. 
WGTL systematically integrates both local and global higher-order graph characteristics.
Witness complex enables us to focus only on the salient shape characteristics delivered by the landmark nodes, thereby reducing the computational costs and minimizing the impact of noisy graph information.

\item We derive the stability guarantees of both local and global topology encodings and the robust topological loss, given an attacker's budget. These guarantees show that local and global encodings are stable to external perturbations, while the stability depends on the goodness of the witness complex construction.

\item Our extensive experiments spanning six datasets and eight GNNs indicate that WGTL boosts the robustness capabilities of GNNs across a wide range of local and global adversarial attacks, 
resulting in relative gains up to 18\%. %\YGL{add gains}.
WGTL also smoothly integrates with other existing defenses, such as Pro-GNN, GNNGuard, and SimP-GCN improving the relative performance up to 4.95\%, 15.67\%, and 5.7\% respectively. 
In addition, WGTL is effective on large-scale and heterophilic graphs, as well as against adaptive and node-feature attacks.
\end{itemize}
\subsection{Related Works: Defenses for GNNs}
% \noindent\textbf{Existing Defenses for GNNs.}
% Works on recent Defense mechanisms and discuss how our approach is fundamentally different from them. 
% \textbf{Adversarial Defenses for GNNs.}
There are broadly three types of defenses: graph purification-based, adversarially robust training, and adversarially robust architecture~\citep{gunnemann2022graph}. 

Notable defenses that purify the input graph include SG-GSR~\citep{in2024self}, Pro-GNN~\citep{prognn} and SVD-GCN~\citep{svdgcn}. These methods learn to remove adversarial edges from the poisoned graph without considering higher-order interactions. In contrast, WGTL primarily focuses on learning the key higher-order graph interactions at both local and global levels and then adaptively assessing their potential defense role via topological regularizer. The local and global topological encodings remain robust despite the false positive edges; as a result, WGTL alleviates the problems associated with false positive edges~\cite{in2024self}, enhancing the overall resilience against attacks.

The adversarial training-based defense methods augment node features with gradients~\cite{flag}, or datasets by generating worst-case perturbations~\citep{pgdattack}. The goal is to train with the worst-case adversarial perturbations such that the learned model weights become more robust against worst-case perturbation~\citep{gunnemann2022graph}. However, adversarial training can not defend against more severe perturbation than the ones they were trained with. 

Better architectures such as VAE~\citep{vgae}, Bayesian uncertainty quantification~\citep{feng2021uag}, and Attention~\citep{zhu2019robust,tang2020} have also been proposed for adversarial defense. 
However, none of these tools have explored the use of robust, higher-order graph topological features as prior knowledge for improved defense. Recently,~\citet{topolayer20} designed a topology-driven attack on images and topological loss, but this approach neither considers graph data nor adversarial defense. 
Among the topology-driven defenses, GNNGuard~\citep{zitnikGNNGuard} discusses graphlet degree vectors %as local graph topology, 
to encode node structural properties such as triangles and betweenness centrality. %Graphlet vector similarity is then used to weigh message passing. 
However, unlike the PH features used in WGTL, the graphlet approach is empirical, without 
%and it is unknown whether graphlet vectors exhibit any 
theoretical robustness guarantees. 
In turn, while the robust loss function has been used before~\citet{zugner2019certifiable}, topological losses have never been used in conjunction with adversarial defenses. The use of a robust loss function as a regularizer for defense is not new; for instance, ~\citet{zugner2019certifiable} proposed \emph{robust hinge loss} for defense. However, it is unknown if topological losses~\cite{hu2019topology,topolayer20} can improve adversarial robustness or not.

\section{Background: Graphs, Persistent Homology, Complexes, Adversarial ML}\label{sec:background}%\vspace*{-0.7em}
\noindent\textbf{Topology of Graphs.} $\mathcal{G} \triangleq (\mathcal{V}, \mathcal{E}, \boldsymbol{X})$ denotes an attributed graph. $\mathcal{V}$ is a set of $N$ nodes. $\mathcal{E}$ is a set of edges. $\boldsymbol{X} \in \mathbb{R}^{N \times F}$ is a node feature matrix, where each node corresponds to an $F$ dimensional feature. The adjacency matrix of $\mathcal{G}$ is  a symmetric matrix $\boldsymbol{A} \in \mathbb{R}^{N \times N}$ such that $\boldsymbol{A}_{uv} \triangleq \omega_{uv}$, i.e., edge weight, if nodes $u$ and $v$ are connected and 0, otherwise. For unweighted graphs, we observe $\omega_{uv}= 1$.
Furthermore, $\boldsymbol{D}$ represents the degree matrix of $\mathcal{G}$, such that $\boldsymbol{D}_{uu} \triangleq \sum_{v \in \mathcal{V}} \boldsymbol{A}_{uv}$ and $0$, otherwise. %while $d_{uv}$ denotes the distance on $\mathcal{G}$, i.e. the shortest path between the nodes $u$ and $v \in \mathcal{V}$. 

% \YGL{Copied from the AAAI2023 paper, needs to be updated so it doesn't look as direct copy-paste}

The central ideas leveraged in this paper are the local and global topology of a graph. The topology of a graph is defined by corresponding geodesic distance. The geodesic distance $d_{\graph}(u,v)$ between a pair of vertices $u$ and $v \in \mathcal{V}$ is defined as the length of the shortest path between $u$ and $v$. The path length is defined as the sum of weights of the edges connecting the vertices $u$ and $v$. Endowed with the canonical metric induced by the geodesic distance $d_{\graph}: \mathcal{V} \times \mathcal{V} \rightarrow \mathbb{R}^{\geq 0}$, a weighted simple graph $\mathcal{G}$ transforms into a metric space $(\mathcal{V},d_{\graph})$. For a given positive real number $\epsilon >0$, the set of nodes that are no more than geodesic $\epsilon$ away from a given node determines the local topology of that node. When $\epsilon = \mathrm{Diam}(\mathcal{G})$, i.e. the diameter of $\mathcal{G}$, we retrieve the global topology of the graph. Increasing $\epsilon$ from $1$ to $\mathrm{Diam}(\mathcal{G})$ allows us to retrieve the evolution of the inherent graph features, like connected components, cycles, voids, etc.~\citep{edelsbrunner2000topological,zomorodian2005topology}. %Connected components, cycles and voids are the dimensions 0,1 and 2 homology features. 
% Broadly speaking, by ``shape" here we mean the 
% properties of the observed object which are preserved under continuous transformations, e.g., stretching, bending, and twisting. (The data can be a graph, a point cloud in Euclidean space, or a sample of points from any metric space). Since one of the most popular PH techniques is to convert the point cloud to a distance graph, for generality we proceed with the further description of PH on graph-structured data. 

\noindent\textbf{Persistent Homology.} To retrieve the evolution of graph features, we employ a Persistent Homology-based approach, a machinery rooted in computational topology. Our topological space originates from subgraphs $\{\mathcal{G}_\alpha: \forall (u,v) \in \mathcal{G}_\alpha, d_{\mathcal{G}}(u,v) \leq \alpha\}_{1 \leq \alpha\leq Diam(\mathcal{G})}$, where every subgraph $\mathcal{G}_\alpha$ contains all edges of length less than $\alpha$. To incorporate higher-order information, simplicial complexes $\{\mathscr{K}(\mathcal{G}_{\alpha})\}_\alpha$ are constructed, where a higher-order simplex $\sigma \in \mathscr{K}(\mathcal{G}_{\alpha})$ if for every node pair $(u,v) \in \sigma$, the simplex $[uv]\in \mathscr{K}(\mathcal{G}_{\alpha})$, in other words $d_\mathcal{G}(u,v) \leq \alpha$. %The sequence of simplicial complexes $\{\mathscr{K}(\mathcal{G}_{\alpha})\}_\alpha$ are nested because the subgraphs are nested, i.e., $\mathcal{G}_{\alpha_i} \subseteq \mathcal{G}_{\alpha_j}$ whenever $\alpha_i \leq \alpha_j$.
This nested sequence of simplicial complexes is called a \textit{filtration}, with $\alpha$ representing the filtration value.

The key idea of PH is to choose multiple scale parameters $\alpha$ and study changes in topological features that occur to $\mathcal{G}$, which evolves with respect to $\alpha$. 
Equipped with the filtration of complexes, we can trace data shape patterns, i.e. the $d$ homology groups, such as independent components, holes, and cavities which appear and merge as scale $\alpha$ changes. For each topological feature $\rho$, we record the indices $b_{\rho}$ and $d_{\rho}$ of $\mathscr{K}(\mathcal{G}_{b_{\rho}})$ and $\mathscr{K}(\mathcal{G}_{d_{\rho}})$, where $\rho$ is first and last observed, respectively. We say that
a pair $(b_{\rho}, d_{\rho})$ represents the birth and death times of $\rho$, and $(d_{\rho} - b_{\rho})$ is its corresponding lifespan (or \textit{persistence}). In general, topological features with longer persistence are considered valuable, while features with shorter persistence are often associated with topological noise. The extracted topological information over the filtration $\{\mathscr{K}_{\alpha_j}\}$ is then represented in $\mathbb{R}^2$ as a {\it Persistence Diagram (PD)}, such that
$\mathcal{\text{PD}}=\{(b_{\rho},d_{\rho}) \in \mathbb{R}^2: d_{\rho} > b_{\rho}\}\cup \Delta$. $\Delta= \{(t, t) | t \in \mathbb{R}\}$ is the diagonal set containing points counted with infinite multiplicity. %Including $\Delta$ allows us to compare different PDs based on the cost of the optimal matching between their points. 
Another useful representation of persistent topological features is \textit{Persistence Image} (PI) that vectorizes the persistence diagram with a Gaussian kernel and a piece-wise linear weighting function~\citep{persistence_images}. Persistence images are deployed to make a classifier ``topology-aware" and are known to be helpful in graph classification~\citep{zhao2019learning,rieck2020uncovering}.
% 
% Our methodology and experimental results shows topology-awareness can improve both the robustness and accuracy of graph classification.

\noindent\textbf{Witness Complexes.} There are multiple ways to construct an abstract simplicial complex $\mathscr{K}$~\citep{zomorodian2005topology}. Due to its computational benefits, one of the widely adopted approaches is a \textit{Vietoris-Rips complex} ($\mathrm{VR}$).
However, the $\mathrm{VR}$ complex uses the entire observed data to describe the underlying topological space, and thus, does not efficiently scale to large and noisy datasets~\citep{zomorodian2010}. 
In contrast, a \textit{witness complex} captures the data shapes using only on a significantly smaller subset $\mathfrak{L}\subseteq \mathcal{V}$, called a set of {\it landmarks}~\citep{witness}. In turn, all other points in $\mathcal{V}$ are used as ``witnesses" that govern the appearances of simplices in the witness complex. \citet{arafat2020epsilon} demonstrate algorithms to construct landmark sets, their computational efficiencies, and stability of the induced \textit{witness complex}. We leverage witness complex to scale to large graph datasets.
\begin{definition}[{Weak Witness Complex}~\citep{witness}]
\label{def2}
We call $w\in \mathcal{V}$ to be a {\it weak witness} for a simplex $\sigma=[v_0 v_1 \ldots v_l]$, where $v_i\in \mathcal{V}$ for $i=0,1,\ldots, l$ and $l \in\mathbb{N}$, with respect to $\mathfrak{L}$ if and only if $d_{\mathcal{G}}(w,v) \leq d_{\mathcal{G}}(w,u)$ for all $v\in \sigma$ and $u \in  \mathfrak{L} \setminus \sigma$. The {\it weak witness complex} $\witness(\mathfrak{L}, \mathcal{G})$ of the graph $\mathcal{G}$ with respect to the landmark set $\mathfrak{L}$  has a node set formed by the landmark points in $\landmarks$, and a subset $\sigma$ of $\mathfrak{L}$ is in $\witness(\mathfrak{L}, \mathcal{G})$ if and only if there exists a corresponding weak witness in the graph $\mathcal{G}$.
\end{definition}
% \vspace*{-.5em}

\noindent\textbf{Adversarial ML and Robust Representations.} Graph Neural Networks (GNNs) aim to learn a labelling function that looks into the features the nodes in the graph $\graph$ and assign one of the $C$ labels $y_v \in \{1,\ldots, C\}$ to each node $v \in \mathcal{V}$~\citep{kipf}. In order to learn to labelling, GNNs often learn compact, low-dimensional representations, aka \textit{embeddings} $R: \graph \times \mathcal{V} \rightarrow \R^{r}$, for nodes that capture the structure of the nodes' neighbourhoods and their features, and then apply a classification rule $f: \R^r \rightarrow \{1,\ldots, C\}$ on the embedding~\citep{kipf,graphsage,zitnikGNNGuard}.

The goal of a robust GNN training mechanism is to learn a labelling function $f\circ R$ such that the change in the predicted labels, i.e. $|(f\circ R)(\graph') - (f\circ R)(\graph)|$, is the minimum, when a graph $\graph$ is adversarially perturbed to become $\graph'$~\citep{zitnikGNNGuard}. The budget of perturbation is defined by $\delta = \|\graph - \graph'\|_p$. $p$ is often fixed to $0$ or $1$~\citep{pgdattack,wu2019adversarial,zugner2019certifiable}.
There are different ways to design a robust training mechanism, such as training with an adversarially robust loss function~\citep{pgdattack}, using a stabilising regularizer to the classification loss~\citep{zugner2019certifiable}, learning a robust representation of the graph~\citep{engstrom2019adversarial,liu2023towards}, etc.

In this paper, we aim to design a \textit{robust representation} $R$ of the graph $\graph$ using its persistent homologies. Specifically, \textit{we call a graph representation $R$ robust,} if for $p, q > 0$,
\begin{align*}
    \|R(\graph) - R(\graph')\|_p = \bigO(\delta),~~\text{when}~\|\graph-\graph'\|_q = \delta.
\end{align*}
In the following section, we propose WGTL, which is a topology-aware graph representation, and show that WGTL achieves this robust representation property.
\section{Learning a Robust Topology-aware Graph Representation}\label{sec:method}
% \vspace*{-.5em}

The general idea is that encoding robust graph structural features as prior knowledge to a graph representation learning framework should induce a degree of robustness against adversarial attacks. 
Graph measures that capture global properties of the graph and measures that rely on aggregated statistics are known to be robust against small perturbations. Examples include degree distribution, clustering coefficients, average path length, diameter, largest eigenvalue and the corresponding eigenvector, and certain centrality measures, e.g., betweenness and closeness centralities. However, these measures are not multiscale in nature. Therefore, they fail to encapsulate global graph structure at multiple levels of granularity. Many of them, e.g., degree distribution, and clustering coefficients, only encode 1-hop or 2-hop information. 
Such information can be learned by a shallow GNN through message passing, rendering such features less useful as a prior. 
Features such as average path length and diameter are too coarse-scale (scalar-valued) and do not help a GNN to discern the nodes. 
Since existing robust graph features can not encode both local and global topological information at multiple scales, we introduce local and global topology encodings based on persistent homology as representations to the GNNs (Section~\ref{sec:wgtl}). We also propose to use a topological loss as a regularizer to learn topological features better (Section~\ref{sec:tloss}).

% \begin{wrapfigure}{r}{0.5\textwidth}
\setlength{\textfloatsep}{4pt}
\begin{figure}
    \centering 
    % \vspace*{-4.5em}
\includegraphics[width=0.47\textwidth]{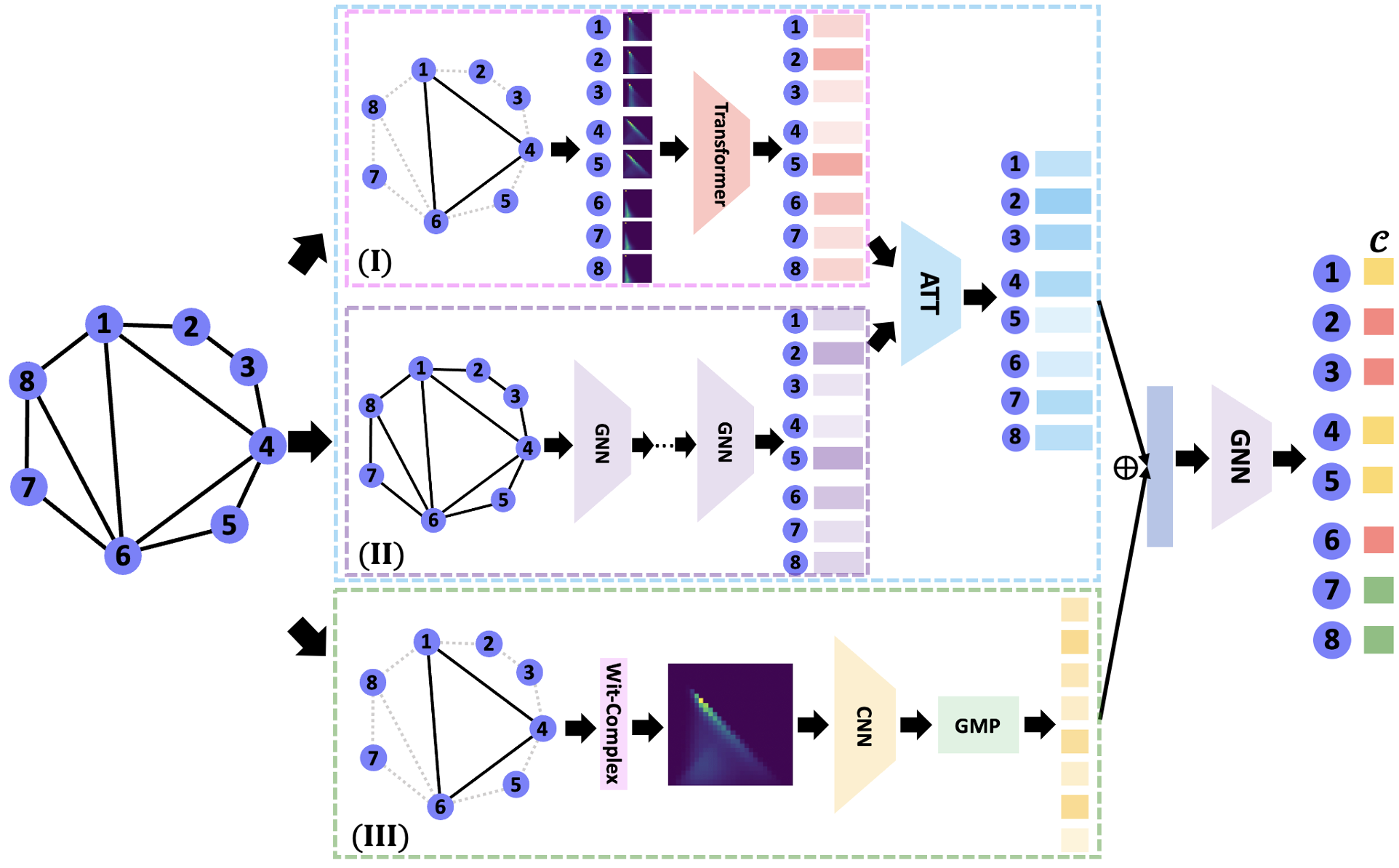}
    \caption{Architecture of Witness Graph Topological Layer.}\label{fig:wgtl}
    % \vspace*{-1.5em}
% \end{wrapfigure}
\end{figure}
% \vspace*{-.5em}

\subsection{Witness Graph Topological Layer}\label{sec:wgtl}
% \vspace*{-.5em}
Now, we describe the architecture of the Witness Graph Topological Layer (WGTL) %The WGTL architecture is illustrated in 
(Figure~\ref{fig:wgtl}).

% the learnt local topological embedding. We refer to 
\noindent\textbf{Component I: Local Topology Encoding.} 
%Component (I) in Figure~\ref{fig:wgtl} illustrates the schematic of the \emph{local topology encoding} component of the WGTL. 
\emph{Local topology encoding} component of WGTL (Figure~\ref{fig:wgtl}) computes local topological features of every node in three steps. 
First, we choose a landmark set $\mathfrak{L}$ from the input graph $\mathcal{G}$.
An important hyperparameter of the local topology encoding is the choice of the number of landmarks. Choosing too few landmarks would reduce the informativeness of the latent embedding. Choosing too many landmarks (i.e., $\left | \mathcal{V} \right |$), on top of being computationally expensive, might be redundant because the topological features of a neighboring node are likely to be the same. 
Secondly, we use the landmarks to construct an $\epsilon$-net of $\mathcal{G}$~\citep{arafat2020epsilon}, i.e. a set of subgraphs $\{\mathcal{G}^{\epsilon}_l\}_{l \in \mathfrak{L}}$. Here, $\epsilon\triangleq \max_{l_1, l_2 \in \landmarks} 0.5 d_{\mathcal{G}}(l_1, l_2)$.
We compute witness complex for each of these $\mathcal{G}^{\epsilon}_l$'s, and the corresponding persistence images $\mathrm{PI}(\witness(\mathcal{G}^{\epsilon}_l))$.
Finally, we attribute the PIs of the landmarks to each node in its $\epsilon$-cover and pass them through a vision transformer model to compute the local topology encoding, i.e. $\boldsymbol{Z}_{T_L} = \mathrm{Transformer}(\mathrm{PI}(\witness(\mathcal{G}^{\epsilon}))_1, \dots, \mathrm{PI}(\witness(\mathcal{G}^{\epsilon}))_N)$. The local topology encoding $\boldsymbol{Z}_{T_L}$ is a latent embedding of local topological features of each node in $\graph$.
% \naheed{We study the sensitivity of the proposed WGTL with respect to the number of landmarks hyperparameter in the appendix.}

When the attack model poisons the adjacency matrix, especially in the cases of global attacks, the local topological encodings are also implicitly perturbed. In Theorem~\ref{thm:stab_local}, we show that local topological encodings are stable w.r.t. perturbations in the input graph. Specifically, if an attacker's budget is $\bigO(\delta)$, the encoded local topology is perturbed by $\bigO( C_{\epsilon} (\delta+ \epsilon))$. The bound indicates the trade-off due to landmark selection. If we select fewer landmarks, computation becomes faster and we encode topological features of a larger neighborhood. However, an increase in $C_{\epsilon}$ yields less stable encoding. Whereas if we select more landmarks, we get more stable encoding but we lose informativeness of the local region and computational efficiency.

\begin{theorem}[Stability of the encoded local topology]\label{thm:stab_local}
Let us denote the persistence diagram obtained from local topology encoding of $\graph$ as $\mathrm{T}(\graph)$ (Figure~\ref{fig:schematic}). For any $p < \infty$ and $C_{\epsilon}$ being the maximum cardinality of the $\epsilon$-neighborhood created by the landmarks, we obtain that for any graph perturbation $\|\mathcal{G}-\mathcal{G}'\|_1 = \bigO(\delta)$ the final persistence diagram representation changes by
$W_p(\mathrm{T}(\mathcal{G}), \mathrm{T}(\mathcal{G}')) = \bigO( C_{\epsilon} \delta)$, if we have access to C\v{e}ch simplicial complexes, and $W_p(\mathrm{T}(\mathcal{G}), \mathrm{T}(\mathcal{G}')) = \bigO( C_{\epsilon} (\delta + \epsilon)),$ if Witness complex is used for the Local Persistence Images.
\end{theorem}
% \begin{corollary}[Stability of the encoded local witness-complex topology]
% For any $p < \infty$ and $C_{\epsilon}$ being the maximum size of the $\epsilon$-net created by the landmarks, we obtain that for any graph perturbation $\|\mathcal{G}-\mathcal{G}'\|_1 = \bigO(\delta)$ the final persistence diagram representation changes by

% \end{corollary}

\textbf{Component II: Graph Representation Learning.} The component II of WGTL deploys in cascade $M$ GNN layers with ReLU activation function and weights $\{\boldsymbol{\Theta}^{(m)}\}_{m=1}^M$. The representation learned at the $m$-th layer is given by  ${\boldsymbol{Z}_{G}^{(m +1)}} =\mathrm{ReLU}(\widetilde{\boldsymbol{{D}}}^{-\frac{1}{2}}\widetilde{\boldsymbol{{A}}}\widetilde{\boldsymbol{{D}}}^{\frac{1}{2}} \boldsymbol{Z}_{G}^{(m)}\boldsymbol{\Theta}^{(m)}).$
Here, $Z_G^{(0)} = \mathcal{G}$, $\widetilde{\boldsymbol{{A}}} = \boldsymbol{A} + \boldsymbol{I}$, 
and $\widetilde{\boldsymbol{{{D}}}}$ is the corresponding degree matrix.

\noindent\textbf{Component III: Global Topology Encoding.} The \textit{global topological encoding} represents the global witness complex-based topological features of a graph (Component III in Figure~\ref{fig:wgtl}). 
First, we use the input adjacency matrix to compute the lengths of all-pair shortest paths (geodesics) among the nodes. The topological space represented by the geodesic distance matrix is used to compute the global witness complex-based persistence image $\mathrm{PI}(\witness(\mathcal{G}))$ of the graph~\citep{arafat2020epsilon}. 
Finally, the persistence image representation is encoded by a Convolutional Neural Network (CNN)-based model to obtain the \textit{global topological encoding} $\boldsymbol{Z}_{T_G} \triangleq \xi_{\text{max}}(\mathrm{CNN}(\mathrm{PI}(\witness(\mathcal{G})))$. Here, $\xi_{\text{max}}(\cdot)$ denotes global max-pooling operation. The global topology encoding encapsulates the global topological features, such as the equivalent class of connected nodes, cycles, and voids in the graph. 

\begin{figure}[t!]
    \centering
    \includegraphics[width=\columnwidth]{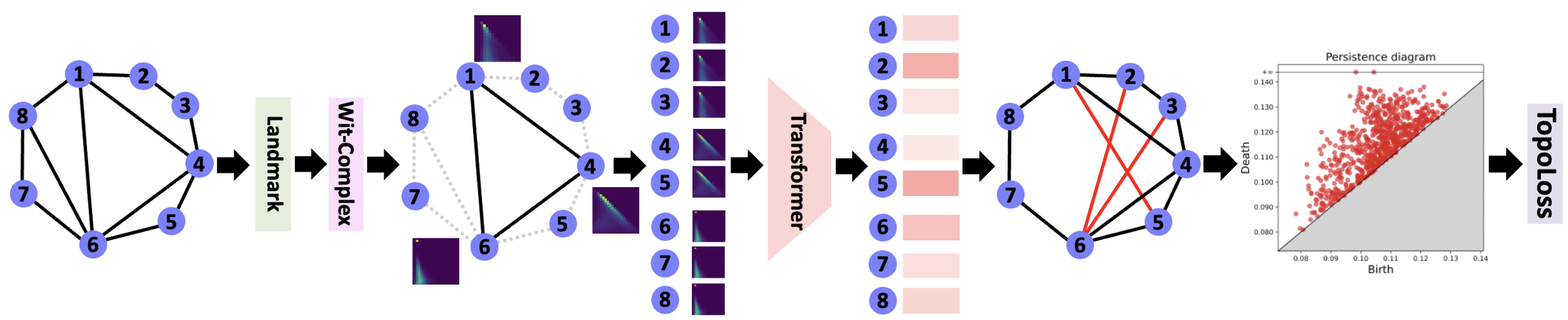}
    \caption{Illustration of Witness Complex-based topological regularizer $L_{Topo}$.}\label{fig:schematic}\vspace*{-.5em}
\end{figure}

The stability of global persistence diagram representation is a well-known classical result in persistence homology~\citep{stability2005,stability2008}. However, given an attacker's budget of $\delta$, the stability of the encoded global topology is an important result for the practical purposes of this paper. Theorem~\ref{thm:stab_global} shows that under a $\bigO(\delta)$ perturbation of the input graph, the global topology encoding is perturbed by $\bigO(\delta+\epsilon)$. 
Hence, the global topological encoding inherits the robustness property of persistent homology and induces robust learning under adversarial attacks.

\begin{proposition}[Stability of the encoded global topology]\label{thm:stab_global}
If the landmarks selected for the witness complex induce an $\epsilon$-net of the graph with $\epsilon > 0$, we obtain that for any graph perturbation $\|\mathcal{G}-\mathcal{G}'\|_1 = \bigO(\delta)$ the global persistence image representation changes by
% \begin{align*}
   $\|\mathrm{PI}(\witness^{\mathrm{glob}}(\mathcal{G})) - \mathrm{PI}(\witness^{\mathrm{glob}}(\mathcal{G}'))\|_{\infty} = \bigO( \delta + \epsilon),$
% \end{align*}
and it reduces to $\bigO( \delta)$, if we have access to the C\v{e}ch simplicial complexes for $\mathcal{G}$.
\end{proposition}

%Global and Local Aggregation
%Topology Prior Aggregation
\noindent\textbf{WGTL: Aggregating Global and Local Encodings.}
We can aggregate the local and global topology encodings with the latent embedding of graph convolution layers in different ways. Figure~\ref{fig:wgtl} shows the approach that empirically provides the most effective defense against adversarial attacks. 
% We have conducted detailed ablation studies in~Appendix~\ref{app:ablation}. 

The aggregation of the three encodings is computed in two steps.
First, to adaptively learn the intrinsic dependencies between learned node embedding and latent local topological encodings, we utilize the attention mechanism to focus on the importance of task-relevant components in the learned representations, i.e. $(\alpha_{G}, \alpha_{T_L}) \triangleq \mathrm{Att}(\boldsymbol{Z}_H, \boldsymbol{Z}_{T_L})$. In practice, we compute attention coefficients as $ \alpha_i = \text{softmax}_i(\Upsilon_{\text{Att}} \tanh{(\boldsymbol{\Xi} \boldsymbol{Z}_i)})$,
% \begin{align*}
%     \alpha_i &= \text{softmax}_i(\Upsilon_{\text{Att}} \tanh{(\boldsymbol{\Xi} \boldsymbol{Z}_i)}) 
%     %= \frac{\exp{(\Upsilon_{\text{Att}} \tanh{(\boldsymbol{\Xi} \boldsymbol{Z}_i)})}}{\sum_{j\in \{G, T_L\}}\exp{(\Upsilon_{\text{Att}} \tanh{(\boldsymbol{\Xi} \boldsymbol{Z}_j)})}},
% \end{align*}
where $\Upsilon_{\text{Att}} \in \mathbb{R}^{1 \times d_{\text{out}}}$ is a linear transformation, $\boldsymbol{\Xi}$ is the trainable weight matrix, and the softmax function is used to normalize the attention vector. Then, we obtain the final embedding by combining two embeddings  
${\boldsymbol{Z}}_{\mathrm{AGG}} = \alpha_G \times \boldsymbol{Z}_G + \alpha_{T_L} \times \boldsymbol{Z}_{T_L}.$
Finally, we combine the learned embedding ${\boldsymbol{Z}}_{\mathrm{AGG}}$ with the latent global topological representation ${\boldsymbol{Z}}_{T_G}$, such that 
%\begin{align}
    $\boldsymbol{Z}_{\mathrm{WGTL}} = {\boldsymbol{Z}}_{\mathrm{AGG}}\boldsymbol{Z}_{T_G}$.
%\end{align}
The node representation ${\boldsymbol{Z}}_{\mathrm{WGTL}}$ encapsulates both global and local topology priors. We call ${\boldsymbol{Z}}_{\mathrm{WGTL}}$ the \emph{aggregated topological priors}.
We feed $\boldsymbol{Z}_{\mathrm{WGTL}}$ into a graph convolutional layer and use a differentiable classifier (here we use a softmax layer) to make node classification.
% \begin{align}
%     {\boldsymbol{Z}}_{\mathrm{WGTL}} = \sigma\left(({\boldsymbol{Z}}_{\mathrm{AGG}}\boldsymbol{Z}_{T_G})\hat{\boldsymbol{A}}\boldsymbol{\Theta}\right),
% \end{align}
% where $\sigma(\cdot)$ is the softmax function, and $\boldsymbol{\Theta}$ denotes the trainable weights.
%\naheed{@Yuzhou, Please write the description of how the attention layer aggregates local and global topology priors. Also, explain why we are only aggregating before the final graph conv rather than doing it earlier graph convs.} 
In the following, we show the stability of the aggregated topological priors.
\begin{proposition}[Stability of the aggregated topological encoding]\label{thm:stab_WGTL}
If the landmarks selected for the witness complex induce an $\epsilon$-net of the graph with $\epsilon > 0$ and $L_{\mathrm{GNN}}$ is the Lipschitz constant of the GNNs in Component II, then for a perturbation $\|\mathcal{G}-\mathcal{G}'\|_1 = \bigO(\delta)$, the encoding ${\boldsymbol{Z}}_{\mathrm{WGTL}}$ changes by
{\small{
\begin{align*}
    \| {\boldsymbol{Z}}_{\mathrm{WGTL}} (\mathcal{G}) - {\boldsymbol{Z}}_{\mathrm{WGTL}} (\mathcal{G}') \|_{1} = 
    \bigO( (C_{\epsilon} + L_{\mathrm{GNN}}) (\delta + \epsilon)^2).
\end{align*}
}}
\end{proposition}
Proposition~\ref{thm:stab_WGTL} shows that the final representations computed by WGTL is stable under adversarial attacks. The stability depends on the approximation trade-off induced by the landmark set and the Lipschitz stability of the GNN layers~\citep{gcnLips}.
%\deb{comment on  corollary}

% \vspace*{-.5em}
\subsection{Topological Loss as a Regularizer}\label{sec:tloss}
% \vspace*{-.5em}%\vspace*{-.5em}%Topological Loss as a regularizer
In Section~\ref{sec:wgtl}, we propose using the aggregated topology encodings to predict node labels for downstream node classification tasks through a GNN backbone. In this case, we use a supervised loss $L_{supv}$ that facilitate learning the aggregated topology priors for classification. We empirically observe that our topology encoding already provides a certain degree of robustness (see ablation studies in Appendix~\ref{app:ablation}).

However, the supervised loss function only explicitly enforces misclassification constraints on the defense model. It does not explicitly enforce any topological constraint such that the topological encodings themselves iteratively become more robust while training. Hence, for increased robustness, we propose to use topological loss $L_{topo}$ that explicitly encodes the birth and death of the topological features in the auxiliary graph (ref. Figure~\ref{fig:schematic}) reconstructed from the transformer output. Specifically, 
\begin{align}\label{eq:topo_loss}
    L_{topo, k}(\mathrm{T}(\mathcal{G})) \triangleq \sum_{i=1}^m (d_i - b_i)^p \left(\frac{d_i+b_i}{2} \right)^{q},
\end{align}
where $m$ is the number of points in the persistence diagram of the auxiliary graph reconstructed from the transformer output and $k=\max\{p,q\}$. In practice, we use $k=2$.
Use of such topological loss was first proposed for image segmentation~\citep{hu2019topology}. \citet{topolayer20} uses it as a regularizer in designing GAN and adversarial attacks on images.
In contrast, we use it to induce stability in the encoding and to defend against adversarial attacks. The benefits of using the topological loss are two-fold:

\noindent (i) \textbf{Persistent and Stable Feature Selection:} Minimising $L_{topo, k}$ causes removal of topological features with smaller persistence, i.e., $(d_i - b_i)$. As such, the regularizer acts as a sparsity-inducing feature selector. By minimizing $L_{topo}$, we are training to learn latent representation such that only the most persistent features remain in the encoded local topology. Such features are known to be more stable and represent more robust structures of the graph.

\noindent (ii) \textbf{Robustness to Local Perturbations:} A localized attack perturbing certain nodes or edges is expected to appear as topological noise in the final persistent diagram, and should exhibit lower persistence.
Since minimizing $L_{topo}$ forces the local topology encodings to eliminate features with small persistences, using $L_{topo}$ as a regularizer with $L_{supv}$ induces robustness to local perturbations in final classification tasks.

Proposition~\ref{prop:stab_tloss} quantifies the stability of the topological regularizer $L_{topo, k}$ under any attack with perturbation budget $\bigO(\delta)$. Specifically, it shows that the stability depends on a trade-off between the maximum persistence of the final graph representation, $A_{\Phi}(\mathcal{G})$, in Figure~\ref{fig:schematic}, and the number of non-zero persistent features in the final encoding. Hence, it reflects our discussion above.
% Intuitively speaking, the attacker's perturbation  Hence, by minimising $L_{topo}$, we are training to learn latent representation such that only the most salient features remain in the encoded local topology. The less salient features will have birth = death; hence, the corresponding persistent pair will be pushed towards the diagonal in the PD. 

% In the following, we prove some results that guarantee the stability of the topology loss $L_{topo}$ and supervised loss $L_{supv}$  under the setting that an attacker has a limited budget of $\bigO(\delta)$. \deb{comment on  corollary}
% \begin{theorem}[Boundedness of $L_{topo}$]
% Let us assume that the smallest triangulation grows polynomially, i.e. $C_{\epsilon} = \bigO(\epsilon^{-M})$ for an $M > 0$. For $k>M$, $ L_{topo, k}(\mathrm{T}(\mathcal{G}))$ is bounded, such that
% \begin{align*}
%     L_{topo, k}(\mathrm{T}(\mathcal{G}))\leq  \left( \sqrt{5}  + \sqrt{\frac{10}{\pi}}\frac{1}{\sigma} \right)^{2k} {\epsilon}^{-2k} \mathrm{Diam}(A(\mathcal{G})) + 2^{k-2} m \mathrm{Diam}(\mathcal{G})^{2k}. 
% \end{align*}
% \end{theorem}
\begin{proposition}[Stability of $L_{topo}$]\label{prop:stab_tloss}
% Let us define
% \begin{align}
%     \mathrm{TotPersLoss}_k(\mathrm{T}(\mathcal{G})) \triangleq \sum_{i=1}^m (d_i - b_i)^p \left(\frac{d_i+b_i}{2} \right)^{q}.
% \end{align}
Let us assume that the cardinality of any $\epsilon$-neighborhood of $\graph$ grows polynomially, i.e. $C_{\epsilon} = \bigO(\epsilon^{-M})$ for an $M > 0$.
If $m$ is the number of points in the persistence diagram, $2k= 2\max\{p,q\} > M$, and $A(\graph)$ is the auxiliary graph constructed from the local topology encodings (Fig.~\ref{fig:schematic}), $L_{topo, k}(\mathrm{T}(\mathcal{G}))$ is stable w.r.t. a perturbation of $\mathcal{G}$, i.e. $\| \mathcal{G} - \mathcal{G'}\|_1 = \delta$.
\begin{align*}
    &\big| L_{topo, k}(\mathrm{T}(\mathcal{G})) - L_{topo, k}(\mathrm{T}(\mathcal{G'}))\big|\\
    =~&\bigO\left(\left({\epsilon}^{-4kM}\mathrm{Diam}(A(\mathcal{G})) + m {\epsilon}^{-2k} \mathrm{Diam}(\mathcal{G})^{2k}\right) \delta \right).
\end{align*}\vspace*{-1.5em}
\end{proposition}
% \deb{sparsity induction -> choice of more stable/persistent reps -> more robustness; expect robustness to local attacks}
% \naheed{\begin{theorem}[Boundedness and stability of the final loss $L$]
% @Deb: TO DO   
% \end{theorem}}

% \vspace*{-1em}
\section{Experimental Evaluation}
% \vspace*{-.5em}
We evaluate the proposed WGTL on the node classification task for clean and attacked graphs across a range of perturbation rates.
We validate the proposed approach on six benchmark datasets: Citeseer, Cora, Pubmed, Polblogs, OGBN-Arxiv, and Snap-patents (ref. Table~\ref{tab:dataset_main}). 
\textit{We report the mean and standard deviation of accuracies over $10$ runs. The best performance is highlighted in bold while the best result on a dataset for a given perturbation rate is indicated by *.}
Note that, throughout our experiments, we use 0-dimensional topological features. All the hyperparameters are chosen by performing cross-validation.
% 
% Our results show that using WGTL and the topological loss as regularizer improve node classification performances across all the scenarios, and thus, validating the utility of robust topological features as priors. 

We defer the dataset descriptions, implementation details, ablation studies, impact of the \#landmarks on performance, comparison with Vietoris-Rips, and additional experimental results such as handling node features, heterophilic graphs, adaptive attacks, and the adoption of other topological vectorization methods in WGTL to the Appendix. 
\setlength\tabcolsep{1.6pt}
\begin{table}[t]
\small
\centering
\begin{tabular}{ccccc}
\toprule
\textbf{Dataset} & \textbf{\begin{tabular}[c]{@{}c@{}}\#nodes (LCC)\end{tabular}} & \textbf{\begin{tabular}[c]{@{}c@{}}\#edges (LCC)\end{tabular}} & \textbf{\#classes} & \textbf{\#features} \\ \hline
Cora-ML          & 2,485                                                            & 5,069                                                            & 7                  & 1,433               \\
Citeseer         & 2,110                                                            & 3,668                                                            & 6                  & 3,703               \\
Pubmed           & 19,717                                                           & 44,338                                                           & 3                  & 500                 \\
Polblogs         & 1,222                                                            & 16,714                                                           & 2                  & None                \\ 
OGBN-arXiv & 169,343 & 1,157,799                                                          & 40                  & 128                \\ 
Snap-patents & 4562 & 12103 & 5 & 269 \\
\bottomrule
\end{tabular}
\caption{Dataset statistics: only the largest connected component (LCC) is considered.}
\label{tab:dataset_main}
\end{table}

\noindent\textbf{Landmark Selection for Local and Global Topology Encodings.} There are several approaches to selecting landmarks, e.g., random selection~\citep{witness}, max-min selection~\citep{witness}, $\epsilon$-net~\citep{arafat2020epsilon} based and centrality-based selection~\citep{topopooling}. In our experiments, we select landmarks based on degree centrality. As shown by~\citet{topopooling}, doing so helps to improve the classification performance. On Cora-ML, Citeseer, and Polblogs, we select 5\% nodes, on Pubmed and Snap-patents we select 2\% nodes and on OGBN-arXiv we select 0.05\% nodes as landmarks to keep \#landmarks roughly invariant across datasets.  

% \setlength{\textfloatsep}{6pt}%
% \begin{wraptable}{r}{0.6\textwidth}
\begin{table}
\centering
% \vspace*{-2.3em}
\setlength\tabcolsep{2pt}
% \small
\resizebox{0.9\columnwidth}{!}{
\begin{tabular}{@{}cccccc@{}}
\toprule
\textbf{Dataset} & \textbf{Models} & \multicolumn{3}{c}{\textbf{Perturbation Rate}} &  \\ \midrule
 &  & 0\% & 5\% & 10\% \\ \midrule
 \multirow{6}{*}{\begin{tabular}[c]{@{}c@{}}Cora-ML\end{tabular}} 
 & Pro-GNN &{82.98$\pm$0.23} & 80.14$\pm$1.34 & 71.59$\pm$1.33 \\
 & Pro-GNN+WGTL & {\bf 83.85$\pm$0.38} & {\bf 81.90$\pm$0.73} & {\bf 72.51$\pm$0.76} \\ 
  \cmidrule(lr){2-5}
   & GCN+GNNGuard & 83.21$\pm$0.34 & 76.57$\pm$0.50 & 69.13$\pm$0.77  \\
 & GCN+GNNGuard+WGTL & $^*${\bf 84.78$\pm$0.43} & $^*${\bf 83.23$\pm$0.82} & $^*${\bf 79.96$\pm$0.49} \\ 
 \cmidrule(lr){2-5}
  & SimP-GCN & 79.52$\pm$1.81 & 74.75$\pm$1.40 & 70.87$\pm$1.70 \\
 & SimP-GCN+WGTL & \bf{81.49$\pm$0.52} &  \bf{ 76.65$\pm$0.65} &  \bf{72.88$\pm$0.83}  \\ 
 \midrule
  \multirow{6}{*}{\begin{tabular}[c]{@{}c@{}}Citeseer\end{tabular}} 
   & ProGNN & 72.34$\pm$0.99 & 68.96$\pm$0.67 & 67.36$\pm$1.12 \\
 & ProGNN+WGTL &{\bf 72.83$\pm$0.94}  & {\bf 71.85$\pm$0.74} & {\bf 70.70$\pm$0.57}\\
 \cmidrule(lr){2-5}
   & GCN+GNNGuard & 71.82$\pm$0.43 & 70.79$\pm$0.22 & 66.86$\pm$0.54 \\
 & GCN+GNNGuard+WGTL & \bf{73.37$\pm$0.63} &  \bf{72.57$\pm$0.17} & \bf{66.93$\pm$0.21} \\ 
 \cmidrule(lr){2-5}
  & SimP-GCN & 73.73$\pm$1.54 & 73.06$\pm$2.09 & 72.51$\pm$1.25 \\
 & SimP-GCN+WGTL & $^*$\bf{74.32$\pm$0.19
} &  $^*$\bf{74.05$\pm$0.71} &  $^*$\bf{73.09$\pm$0.50} \\ 
 \midrule
  \multirow{6}{*}{\begin{tabular}[c]{@{}c@{}}Pubmed\end{tabular}} 
   & Pro-GNN & 87.33$\pm$0.18 & 87.25$\pm$0.09 & 87.20$\pm$0.12\\
& Pro-GNN + WGTL (ours) &{\bf 87.90$\pm$0.30} & $^*${\bf 87.77$\pm$0.08} & $^*${\bf 87.67$\pm$0.22}\\
 \cmidrule(lr){2-5}
   & GCN+GNNGuard & 83.63$\pm$0.08 & 79.02$\pm$0.14 & 76.58$\pm$0.16 \\
 & GCN+GNNGuard+WGTL & OOM & OOM & OOM \\ 
 \cmidrule(lr){2-5}
  & SimP-GCN & $^*$88.11$\pm$0.10 & 86.98$\pm$0.19 & 86.30$\pm$0.28 \\
 & SimP-GCN+WGTL & OOM & OOM & OOM \\ 
  \midrule
  \multirow{4}{*}{\begin{tabular}[c]{@{}c@{}}Polblogs\end{tabular}} 
 %   & ProGNN & - & - & - \\
 % & Pro-GNN + WGTL (ours) & - & - & - \\
 % \cmidrule(lr){2-5}
   & GCN+GNNGuard & 95.03$\pm$0.25 & 73.25 $\pm$0.16 & 72.76$\pm$0.75 \\
 & GCN+GNNGuard+WGTL & $^*$\bf{96.22$\pm$0.25} & $^*$\bf{73.62$\pm$0.22} & $^*$\bf{73.72$\pm$1.00} \\ 
 \cmidrule(lr){2-5}
  & SimP-GCN & 89.78$\pm$6.47 & 65.75$\pm$5.03 & 61.53$\pm$6.41 \\
 & SimP-GCN+WGTL & \bf{94.56$\pm$0.24} & \bf{69.78$\pm$4.10} & \bf{69.55$\pm$4.42} \\ 
\bottomrule
\end{tabular}%
}
% \vspace*{-1em}
% \end{wraptable}
\caption{Comparison of performances (avg. accuracy$\pm$std.)
with existing defenses under mettack.}
\label{prognn_mettatack_results}
\end{table}

\noindent\textbf{Adversarial Attacks: Local and Global.} 
We deploy four local and global poisoning attacks, with perturbation rates, i.e., the ratio of changed edges, from 0\% to 10\%, to evaluate the robustness of WGTL. We consider a fixed GCN without weight re-training as the surrogate for all attacks. 
As a local attack, we deploy nettack~\citep{nettack}. 
% Nettack is a targeted attack that selects nodes without violating the degree distribution and feature co-occurrence of the original graph, and then, perturbs the edges around them.
\textit{Due to the stability of WGTL and topological regularizer, we expect to be robust to such local attacks.}
As global (non-targeted) poisoning attacks, we deploy mettack~\citep{mettack}, and two topological attacks, namely PGD~\citep{pgdattack} and Meta-PGD~\citep{metapgd}.  Mettack treats the graph as a hyperparameter and greedily selects perturbations based on meta-gradient for node pairs until the budget is exhausted. We keep all the default parameter settings (e.g., $\lambda=0$) following the original implementation~\citep{mettack}. For Cora-ML, Citeseer and Polblogs, we apply the most effective Meta-Self variant, while for Pubmed, we apply the approximate variant (A-Meta-Self) to save memory and time~\citep{prognn}.
PGD attack~\citep{pgdattack} adapts the well-known Projected Gradient Descent-based attack in adversarial ML to graphs. Recently, Meta-PGD is proposed~\citep{metapgd} by applying PGD on the meta-gradients. It combines the effectiveness of mettack and PGD, and is shown to be the most effective topological attack at present. 
\textit{Though global attacks are expected to be more challenging while using topological features, we demonstrate that WGTL still yields significant robustness.}
Further details on attack implementations and attackers' budgets are in Appendix~\ref{app:attacks}. The results for PGD and Meta-PGD attacks are in Appendix~\ref{app:moreresults}.

% \setlength{\textfloatsep}{6pt}%
% \begin{table}[t!]
% % \begin{wraptable}[11]{r}{0.7\textwidth}
% % \vspace*{-7mm} % keep vspace here
% \centering
% \setlength\tabcolsep{2pt}
% % \footnotesize % font smaller
% \caption{Node classification performance (Accuracy$\pm$Std) under mettack. Pro-GNN is used as a backbone GNN architecture.\label{prognn_mettatack_results}}
% \resizebox{\columnwidth}{!}{
% \begin{tabular}{llccc}
% \toprule
% \multirow{2}{*}{\textbf{Dataset}}&\multirow{2}{*}{\textbf{Model}}& \multicolumn{3}{c}{\textbf{Perturbation Rate}}
% \\
% \cmidrule(lr){3-5}%\cmidrule(lr){8-13}
%            & & 0\% & 5\% & 10\%\\
% \midrule
% \multirow{2}{*}{{Cora-ML}} & Pro-GNN &{82.98$\pm$0.23} & 80.14$\pm$1.34 & 71.59$\pm$1.33\\
% & Pro-GNN + WGTL (ours) & $^*${\bf 83.85$\pm$0.38} & $^*${\bf 81.90$\pm$0.73} & $^*${\bf 72.51$\pm$0.76}\\
% \hline
% \multirow{2}{*}{{Citeseer}}& Pro-GNN & 72.34$\pm$0.99 & 68.96$\pm$0.67 & 67.36$\pm$1.12\\
% & Pro-GNN + WGTL (ours) &$^*${\bf 72.83$\pm$0.94}  & $^*${\bf 71.85$\pm$0.74} & $^*${\bf 70.70$\pm$0.57}\\
% \hline
% \multirow{2}{*}{{Pubmed}}& Pro-GNN & 87.33$\pm$0.18 & 87.25$\pm$0.09 & 87.20$\pm$0.12\\
% & Pro-GNN + WGTL (ours) &$^*${\bf 87.90$\pm$0.30} & $^*${\bf 87.77$\pm$0.08} & $^*${\bf 87.67$\pm$0.22}\\
% \hline 
% \multirow{2}{*}{{Polblogs}}& Pro-GNN & - & - & - \\
% & Pro-GNN + WGTL (ours) & - & - & - \\
% \bottomrule
% \end{tabular}}%\vspace*{4em}
% \end{table}
% \end{wraptable}

\textbf{Objectives.} \emph{We implemented and compared WGTL with 3 existing defenses and 5 GNN backbones to study five questions. (Q1) Can WGTL enhance the robustness of the existing defenses? (Q2) Can WGTL enhance the robustness of existing backbone graph convolution layers? (Q3) Is WGTL still effective when the topological features are computed on poisoned graphs instead of clean graphs? (Q4) How does WGTL perform on large graphs? (Q5) Is WGTL computationally efficient? 
}
\setlength{\textfloatsep}{4pt}
\begin{table*}
\begin{minipage}{0.46\linewidth}
    \centering
\resizebox{\columnwidth}{!}{
\begin{tabular}{@{}ccccc@{}}
\toprule
\textbf{Dataset}& \textbf{Model}& \multicolumn{3}{c}{\textbf{Perturbation Rate}}
\\\midrule
           & & 0\% & 5\% & 10\% \\
\midrule
\multirow{8}{*}{{Cora-ML}} & GCN &82.87$\pm$0.83 &76.55$\pm$0.79 &70.39$\pm$1.28 \\
& GCN + WGTL (ours) & {\bf 83.83$\pm$0.55} & {\bf 78.63$\pm$0.76} & {\bf 73.41$\pm$0.82}  \\
\cmidrule(lr){2-5}
& ChebNet & 80.74$\pm$0.42 & 74.35$\pm$1.2 &	66.62$\pm$1.44 \\
& ChebNet + WGTL (ours) & {\bf 82.96$\pm$1.08} & {\bf 76.00$\pm$1.22}& {\bf 69.49$\pm$0.89} \\
\cmidrule(lr){2-5}
& GAT & 84.25$\pm$0.67	 & 79.88$\pm$1.09 & 72.63$\pm$1.56 \\
& GAT + WGTL (ours) & $^*${\bf 86.07$\pm$2.10} & {\bf 80.80$\pm$0.87} & {\bf 75.80$\pm$0.79} \\
\cmidrule(lr){2-5}
& GraphSAGE & 81.00$\pm$0.27 & 74.81$\pm$1.2 &	70.92$\pm$1.18 \\
& GraphSAGE + WGTL (ours) & {\bf 83.63$\pm$0.35} & $^*${\bf 82.61$\pm$0.65}& $^*${\bf 81.19$\pm$1.13}  \\
\hline 
\multirow{8}{*}{{Polblogs}}& GCN & 94.40$\pm$1.47 & 71.41$\pm$2.42 & 69.16$\pm$1.86 \\
& GCN + WGTL (ours) & $^*${\bf 95.95$\pm$0.15
} & {\bf74.62$\pm$0.42} & {\bf72.84$\pm$0.86} \\
\cmidrule(lr){2-5}
& ChebNet & 73.10$\pm$7.13 &  67.63$\pm$1.71 &	67.36$\pm$0.85 \\
&ChebNet + WGTL (ours) & {\bf 92.50$\pm$1.10} & {\bf 71.17$\pm$0.10} & {\bf 68.03$\pm$0.87} \\
\cmidrule(lr){2-5}
& GAT & 95.28$\pm$0.51 & 75.83$\pm$0.90 & 73.11$\pm$1.20 \\
&GAT + WGTL (ours) &{\bf 95.87$\pm$0.26} & $^*${\bf 83.13$\pm$0.32} & {\bf 80.06$\pm$0.50} \\
\cmidrule(lr){2-5}
& GraphSAGE & 94.52$\pm$0.27&	77.44 $\pm$ 1.71 &	74.66$\pm$0.85  \\
&GraphSAGE + WGTL (ours) & {\bf 95.58$\pm$0.50} & {\bf 82.62$\pm$0.65} & $^*${\bf 81.49$\pm$0.86} \\
\bottomrule
\end{tabular}
}
\caption{Robustness of various backbone GNNs (avg. accuracy$\pm$std.) under mettack}
\label{mettack_results}
\end{minipage}\hfill
\begin{minipage}{0.46\linewidth}
    \centering
\resizebox{\columnwidth}{!}{%
\begin{tabular}{@{}cccccc@{}}
\toprule
\textbf{Dataset} & \textbf{Models} & \multicolumn{3}{c}{\textbf{Perturbation Rate}} &  \\ \midrule
 &  & 0\% & 5\% & 10\% \\ \midrule
 \multirow{10}{*}{\begin{tabular}[c]{@{}c@{}}Cora-ML\end{tabular}} 
 & GCN & 82.87$\pm$0.83 & 76.55$\pm$0.79 & 70.39$\pm$1.28 \\
 & GCN+WGTL$_{\mathrm{P}}$(ours) & \bf{83.83$\pm$0.55} & \bf{76.96$\pm$0.76} & \bf{71.31$\pm$0.85} \\ 
 \cmidrule(lr){2-5}
  & GAT & 84.25$\pm$0.67 & 79.88$\pm$1.09 & 72.63$\pm$1.56 \\
 & GAT+WGTL$_{\mathrm{P}}$(ours) & $^*$\bf{86.07$\pm$2.10} & \bf{81.43$\pm$0.75} & $^*$\bf{73.74$\pm$1.92} \\ 
  \cmidrule(lr){2-5}
   & GraphSAGE & 81.00$\pm$0.27 & 74.81$\pm$1.20 & 70.92$\pm$1.18 \\
 & GraphSAGE+WGTL$_{\mathrm{P}}$(ours) & \bf{83.63$\pm$0.35} & $^*$\bf{82.15$\pm$1.25} & \bf{73.57$\pm$0.73} \\ 
\cmidrule(lr){2-5}
   & ProGNN & 82.98$\pm$0.23 & 80.14$\pm$1.34 & 71.59$\pm$1.33 \\
 & ProGNN+WGTL$_{\mathrm{P}}$(ours) & \bf{83.85$\pm$0.38} & \bf{81.69$\pm$1.83} & \bf{72.71$\pm$1.26} \\ 
 \cmidrule(lr){2-5}
   & GCN+GNNGuard & 83.21$\pm$0.34 & 76.57$\pm$0.50 & 69.13$\pm$0.77  \\
 & GCN+GNNGuard+WGTL$_{\mathrm{P}}$(ours) & \bf{84.78$\pm$0.43} & \bf{77.08$\pm$0.32} & \bf{70.15$\pm$0.89} \\ 
 \midrule
  \multirow{8}{*}{\begin{tabular}[c]{@{}c@{}}Polblogs\end{tabular}} 
 & GCN & 94.40$\pm$1.47 & 71.41$\pm$2.42 & 69.16$\pm$1.86 \\
 & GCN+WGTL$_{\mathrm{P}}$(ours) & \bf{95.95$\pm$0.15} & \bf{73.02$\pm$1.13} & \bf{74.52$\pm$0.28} \\ 
 \cmidrule(lr){2-5}
  & GAT & 95.28$\pm$0.51 & 75.83$\pm$0.90 & 73.11$\pm$1.20 \\
 & GAT+WGTL$_{\mathrm{P}}$(ours) & \bf{95.87$\pm$0.26} & \bf{76.05$\pm$0.79} & \bf{74.21$\pm$0.74} \\ 
  \cmidrule(lr){2-5}
   & GraphSAGE & 94.54$\pm$0.27 & 77.44$\pm$1.71 & 74.66$\pm$0.85 \\
 & GraphSAGE+WGTL$_{\mathrm{P}}$(ours) & \bf{95.58$\pm$0.50} & $^*$\bf{78.65$\pm$1.32} & $^*$\bf{74.93$\pm$0.81} \\ 
% \cmidrule(lr){2-5}
%    & ProGNN & - & - & - \\
% & ProGNN+WGTL$_{\mathrm{P}}$ & - & - & - \\ 
 \cmidrule(lr){2-5}
   & GCN+GNNGuard & 95.03$\pm$0.25 & 73.25 $\pm$0.16 & 72.76$\pm$0.75 \\
 & GCN+GNNGuard+WGTL$_{\mathrm{P}}$(ours) & $^*$\bf{96.22$\pm$0.25} & \bf{73.62$\pm$0.22} & \bf{73.72$\pm$1.00} \\ 
 \bottomrule
\end{tabular}%
}
\caption{Performance on poisoned graph (avg. accuracy$\pm$std) under mettack.} %WGTL$_{\mathrm{P}}$ fetches features on poisoned graph}
\label{results_poisoned}
\end{minipage}
% \vspace*{-1.5em}
\end{table*}

\noindent\textbf{Q1. Performance of WGTL Defense w.r.t. Existing Defenses.} %\vspace*{-.5em} %Effectiveness of WGTL
%We evaluate the robustness of the proposed WGTL by evaluating GCN+WGTL and Pro-GNN+WGTL implemented over the GCN and Pro-GNN backbones, respectively. 
% \naheed{Experimental results with additional backbone GNN methods are presented in the appendix.}
%%%
We compare our method with three state-of-the-art defenses: Pro-GNN~\cite{prognn}, GNNGuard~\cite{zitnikGNNGuard}, and SimP-GCN~\cite{simpgcn}. 
Table~\ref{prognn_mettatack_results} illustrates the comparative performances on three citation networks under a global attack, i.e. mettack. We observe that our Pro-GNN+WGTL is always better than other baselines on all datasets. Following~\citet{prognn}, we omit Pro-GNN for Polblogs. As a consequence, we gain 0.68\% - 4.96\% of relative improvements on Cora-ML and Citeseer. Similarly, we observe that GCN+GNNGuard+WGTL outperforms GCN+GNNGuard and SimP-GCN+WGTL outperforms Simp-GCN by 0.10\% - 15.67\% and 2.4\% - 5.7\%, respectively, across all datasets. The results reveal that WGTL enhances not only model expressiveness but also the robustness of the GNN-based models. The performance comparison under nettack is shown in Table~\ref{SOTA_nettack} of Appendix~\ref{app:moreresults}.

\noindent\textbf{Q2. WGTL Enhances Robustness of GNNs.} WGTL is flexible in the sense that it can employ existing GNN layers to enhance their robustness. To be precise, we have employed the existing GNN backbones as component II in Figure~\ref{fig:wgtl} to enhance their robustness.
% 
% Table~\ref{nettack_results} and Table~\ref{mettatack_results} present the performances of WGTL on four backbones, i.e. vanilla GCN~\cite{kipf}, ChebNet~\cite{chebnet2016}, Graph attention (GAT)~\cite{gat}, and GraphSAGE~\cite{graphsage}, on various datasets under mettack. 
% We observe three things. (i) Nettack changes the node connections while preserving the degree distribution. Thus, it stealthily changes the graph topology. Hence, intuitively nettack should diminish the informativeness of the local and global topology priors on the perturbed graphs. Despite that, we observe that WGTL still consistently improves the robustness of all backbone GNNs in Table~\ref{nettack_results}. 
Since global attacks target global graph topology, global poisoning attacks are supposed to be more challenging for the proposed {\it topology-based} defense WGTL. Despite that, we observe that WGTL consistently improves the robustness of all backbone GNNs in Table~\ref{mettack_results}. 
% Hence, all the WGTL models have a lower accuracy on a given dataset when encountering mettack compared to nettack. (iii) 
The performance of our method, including that of the backbone GNNs, deteriorates faster on Polblogs than on the other datasets. This is because Polblogs does not have node features, and having informative node features can help GNN differentiate between nodes and learn meaningful representations despite changes in the graph structure. 
With node features lacking, the Polblogs graph has comparatively less resilience against graph structural perturbations. We discuss the results on Nettack and results with SGC backbone~\cite{sgc2019} in  Appendix~\ref{app:moreresults}.

\noindent\textbf{Q3. Performance of WGTL on Poisoned Graphs.} So far, the local and global topological features are computed on clean graphs assuming that these features can be computed before the attacker poisons the graph. However, such an assumption is restrictive as the attacker might poison the graph at any point before and during training. As a result, the topological features computed by WGTL might also be poisoned, as they were computed based on the poisoned graph. WGTL$_{\mathrm{P}}$ employs poisoned graphs as inputs in the schematics of Figure~\ref{fig:wgtl} and~\ref{fig:schematic}. We present the performance of WGTL$_{\mathrm{P}}$ on Cora-ML and Polblogs under mettack in Table~\ref{results_poisoned}. We observe a consistent improvement over the baseline models across various datasets and perturbation rates. In this setting, we find GAT+WGTL$_\mathrm{P}$ and GraphSAGE+WGTL$_\mathrm{P}$ to be the best performing models. We observe that \textit{WGTL robustifies the existing backbones, e.g., GAT and SAGE, more compared to all other defenses}.

\begin{table}[t]
\setlength\tabcolsep{1pt}
% \smaller 
\centering
\resizebox{0.6\columnwidth}{!}{
\begin{tabular}{ccc}
\toprule
\textbf{Models} & \multicolumn{2}{c}{\textbf{Perturbation Rate}}  \\ 
\midrule
   & 0\%	& 10\% \\
\midrule
GCN	& 27.33 & 21.56 \\
GCN+WGTL$_{\mathrm{P}}$ (ours) &\textbf{28.32} & \textbf{22.89} \\
\bottomrule
\end{tabular}
}
\caption{Performance on OGBN-arXiv under PRBCD attack.}\label{tab:ogbn}
\end{table}
% \noindent\textbf{Q4. Performance on adaptive attacks.} Table~\ref{tab:adaptive} in Appendix~\ref{app:moreresults} shows that our method improves the robustness of GCN by 0.4\% and 9.6\% on Cora-ML under evasion and poisoning type Aux-attack~\cite{metapgd} respectively, while on Polblogs the improvements are 1.8\% and 21.9\%.

\noindent\textbf{Q4. Performance of WGTL on a large-scale graph.}
We have applied the PRBCD attack to generate the perturbed OGBN-arXiv graph since we found that other attacks, such as Mettack and Nettack, do not scale to such large-scale graphs~\cite{prbcd}. Following~\citet{prbcd}, we train a 3-layer GCN to generate attacks. We then present the comparison between GCN and GCN + WGTL$_{\mathrm{P}}$ on non-poisoned ($0\%$) and poisoned ($10\%$) perturbed graphs in Table~\ref{tab:ogbn}. 
% The GCN is run with early stopping for $3000$ epochs and patience of $300$. We have chosen $0.0005\%$ nodes landmarks for witness feature computation, resulting in $84$ landmarks. 
% Table~\ref{tab:ogbn} depicts the results for node classification performance comparison between GCN and our GCN + WGTL$_{\mathrm{P}}$ on perturbed OGBN-arXiv graph. 
We observe that the GCN equipped with our WGTL outperforms GCN on both clean and perturbed OGBN-arXiv. 

\noindent\textbf{Q5. Computational Complexity and Efficiency of WGTL.}
Landmark selection (top-$|\mathfrak{L}|$ degree nodes) has complexity $\mathcal{O}(N\log(N))$. To compute witness features, we compute (1) landmarks-to-witness distances costing $\mathcal{O}(|\mathfrak{L}|(N+|\mathcal{E}|))$ due to BFS-traversal from landmarks, (2) landmark-to-landmark distances costing $\mathcal{O}(|\mathfrak{L}|^2)$, and finally (3) persistent homology via boundary matrix construction and reduction~\citep{edelsbrunner2000topological}. Matrix reduction algorithm costs $\mathcal{O}(\zeta^3)$, where $\zeta$ is the \#simplices in a filtration. 
% For 0-dimensional PH, it can be computed efficiently using disjoint sets with complexity $\mathcal{O}(\zeta\alpha^{-1}(\zeta))$, where $\alpha^{-1}(\cdot)$ is the inverse Ackermann function~\citep{cormen2022introduction}. 
Overall computational complexity of computing witness topological feature on a graph is $\mathcal{O}(|\mathfrak{L}|(N+|\mathcal{E}|)+|\mathfrak{L}|^2+\zeta^3)$.
% The computational complexity of the witness complex construction is $\mathcal{O}(|\mathfrak{L}|\log(N))$, while including computation of the distances between data points and landmark points.
% Here, $N$ is the number of data points and $\mathfrak{L}$ is the landmark set.
%O(L log (n)) (where n is the number of data points and L is the landmark set), involving calculating the distance between data points and landmark points.
% We observe that the average running times (i.e., training time per epoch) of WGTL (using GCN as the backbone architecture) on Cora-ML, Citeseer, Pubmed, and Polblogs are 3.86 s, 2.72 s, 6.97 s, and 0.78 s, respectively.

Table~\ref{tab:time} shows the total CPU-time to compute Witness topological features broken down into the time spent to select landmarks, to compute local and global topological features. We find that on all the graphs except Pubmed and OGBN-arXiv, the total computation time is $< 9$ seconds. On Pubmed, it takes $\sim$28 seconds, and on large-scale graph OGBN-arXiv, it takes $\sim$96 seconds. These run times are practical given the scale of these graphs (see Table~\ref{tab:dataset}, Appendix~\ref{app:details}). 

It is worth noticing that landmark selection is quite efficient, hence not a bottleneck. However, the most expensive part of the computation is the computation of global topological features. We empirically found that this is due to the BFS traversal from $|\mathfrak{L}|$ landmarks. One can reduce this cost by selecting fewer landmarks. Another option is to adopt libraries such as RAPIDS cuGraph~\cite{cugraph} which uses GPU acceleration for BFS traversal. However, we found negligible improvement in the overall run-time. 
% e need to transfer the computed results from GPU to CPU (for the downstream PH computation) and such transfer has some additional overhead. 
% 
% \setlength\tabcolsep{1pt}
\begin{table}[t]
\small
\resizebox{\columnwidth}{!}{
\begin{tabular}{@{}cccc@{}}
\toprule
\textbf{\begin{tabular}[c]{@{}c@{}}Datasets/\\\# Landmarks\end{tabular}} & \textbf{\begin{tabular}[c]{@{}c@{}}Landmark \\selection time \end{tabular}} & \textbf{\begin{tabular}[c]{@{}c@{}}Local feat.\\ comput. time\end{tabular}} & \textbf{\begin{tabular}[c]{@{}c@{}}Global feat.\\comput. time\end{tabular}}  \\ \midrule
Cora-ML/124 & 0.01$\pm$0.01 & 0.12$\pm$0.03 & 5.11$\pm$0.13 \\
Citeseer/105 & 0.01$\pm$0.01 & 0.16$\pm$0.02 & 5.23$\pm$1.22 \\
Polblogs/61 & 0.01$\pm$0.00 & 0.07$\pm$0.01 & 4.64$\pm$0.2\\
Snap-patents/91 & 0.03$\pm$0.02 & 0.64$\pm$0.00 & 7.54$\pm$1.15 \\
Pubmed/394 & 0.07$\pm$0.01 & 0.51$\pm$0.03 & 27.83$\pm$0.47 \\
OGBN-arXiv/84 & 1.02 $\pm$0.00 & 12.79$\pm$0.31 & 83.04$\pm$2.19 \\
\bottomrule
\end{tabular}
}
\caption{\label{tab:time} Efficiency of WGTL. All the times are in seconds.}
% \vspace*{-1em}
\end{table}

% \vspace*{-1.5em}
\section{Conclusion and Future Works}
% \vspace*{-1em}
By harnessing the strengths of witness complex to efficiently learn topological representations based on the subset of the most essential nodes as skeleton, we have proposed the novel topological defense against adversarial attacks on graphs,
WGTL.
WGTL is versatile and can be readily integrated with any GNN architecture or another non-topological defense, leading to substantial gains in robustness. We have derived theoretical properties of WGTL, both at the local and global levels, and have illustrated its utility across a wide range of adversarial attacks.

In future, we plan to explore the utility of WGTL with respect to adversarial learning
of time-evolving graphs and hypergraphs. Another interesting research direction is to investigate the linkage between the attacker's budget, number of landmarks, and topological attacks targeting the skeleton shape, that is, topological properties of the graph induced by the most important nodes (landmarks).

\section*{Acknowledgement}
This work was supported by the NSF grant TIP-2333703 and the ONR grant 
N00014-21-1-2530. Also, the paper is based upon work supported by (while Y.R.G. was serving at) the NSF. The views expressed in the article do not necessarily represent the views of NSF or ONR. D. Basu acknowledges the ANR JCJC project REPUBLIC (ANR-22-CE23-0003-01), the PEPR project FOUNDRY (ANR23-PEIA-0003), and the CHIST-ERA project CausalXRL (ANR-21-CHR4-0007).

\bibliography{icml_refs}

% \clearpage

% \input{appendix/checklist}

\clearpage
\appendix

\section*{Appendix}
% \section{Limitations and Broader Impacts}
% \label{sec:limitation}
% The accuracy and execution times reported in this paper may vary depending on the number and choice of landmarks, and the machine used to train the models. In this work, we only considered poisoning-type attacks, other attack types such as node-injection attacks and attacks on node-embeddings are planned for future works. Finally, it is important for both
% researchers and investigators to exercise prudent
% judgment and validate findings through additional comparative
% methods before arriving at any definitive conclusions or undertaking consequential actions.
% 
\section{Theoretical Analysis}
\label{sec:theory}
% \onecolumn
\begin{table*}[htb]
\centering
\resizebox{\textwidth}{!}{
\begin{tabular}{l l l}
\toprule
 $\graph$ &$\triangleq$ & A graph with a vertex set $\nodes$, an edge set $\edges$, and features $\boldsymbol{X}$\\
 $N$ &$\triangleq$ & Cardinality of $\nodes$, i.e. the number of nodes\\
 $F$ &$\triangleq$ & Dimension of features corresponding to each node\\
 $\boldsymbol{A}$ &$\triangleq$ & Adjacency matrix of $\graph$\\
 $\boldsymbol{D}$ &$\triangleq$ & Degree matrix of $\graph$\\
 $d_{\graph}(u,v)$ &$\triangleq$ & Geodesic distance between nodes $u$ and $v$ in graph $\graph$\\
 $\mathrm{Diam}(\graph)$ &$\triangleq$ & The diameter of the graph $\graph$\\
 $\landmarks$ &$\triangleq$ & The set of landmark nodes\\
 $\epsilon$ &$\triangleq$ & Radius of the $\epsilon$-net induced on $\graph$ by $\landmarks$\\
 $\graph^{\epsilon}_l$ &$\triangleq$ & The $\epsilon$-neighborhood of the landmark $l$ in graph $\graph$\\
 $C_{\epsilon}$ &$\triangleq$ & Maximum cardinality of the $\epsilon$-neighborhoods induced by the landmarks $\landmarks$\\
 $d_H(\graph_1, \graph_2)$ &$\triangleq$ & Hausdroff distance between graphs $\graph_1$ and $\graph_2$\\
 $W_p(\cdot, \cdot)$ &$\triangleq$ & $p$-Wasserstein distance\\
  $\mathrm{PD}(\mathscr{K}(\graph))$ &$\triangleq$ & Persistence diagram of the $\mathscr{K}$ simplicial complexes computed on a graph $\graph$.\\
  & & $\mathscr{K}$ can be C\v{e}ch, Vietoris-Rips or Witness complex of dimension $d \in \mathbb{Z}_{\geq 0}$.\\
  $\mathrm{PI}(\mathscr{K}(\graph))$ &$\triangleq$ & Persistence image of the $\mathscr{K}$ simplicial complexes computed on a graph $\graph$.\\
  & & In our analysis, $\mathcal{K}$ can be C\v{e}ch, Vietoris-Rips or Witness complex of $d \in \mathbb{Z}_{\geq 0}$.\\
  $\boldsymbol{A}(\graph) $&$\triangleq$ & The adjacency matrix constructed from the local topology encoding $\boldsymbol{Z}_{T_L}$ of the nodes\\
  $\mathrm{T}(\graph)$ &$\triangleq$ & Persistence diagrams of dimension $d \in \mathbb{Z}_{\geq 0}$ constructed from $A(\graph)$ \\
  $\boldsymbol{A}^{glob}(\graph) $&$\triangleq$ & The adjacency matrix constructed from the global topology encoding $\boldsymbol{Z}_{T_G}$ of the nodes\\
  $\gaussconst$ &$\triangleq$ &
  constant introduced due to kernelization of the persistent images, e.g., for Gaussian kernels $C_{\sigma}=\left( \sqrt{5}  + \sqrt{\frac{10}{\pi}}\frac{1}{\sigma} \right)$.
  
  \\
%  &$\triangleq$ & \\
%  &$\triangleq$ & \\
%  &$\triangleq$ & \\
 \bottomrule
\end{tabular}
}
\caption{Notations.}\label{tab:Notation}%\\
\end{table*}
\subsection{Notations}\label{app:notations}
We dedicate Table~\ref{tab:Notation} to index the notations used in this paper. Note that every notation is also defined when it is introduced.
\begin{definition}[Hausdroff distance~\citep{hausdorff1914}]
    Let $(S, d)$ be a metric space with metric $d: S \times S \gets \mathbb{R}$. Then, Hausdroff distance between any two non-empty subsets $X$ and $Y$ of $S$ is defined as
    \begin{align*}
        d_H(X, Y) \triangleq \max \bigg\lbrace \sup_{x\in X} \inf_{y \in Y} d(x,y), \sup_{y\in Y} \inf_{x \in X} d(x,y) \bigg\rbrace.
    \end{align*}
\end{definition}

\begin{definition}[$p$-Wasserstein distance~\citep{pwasserstein}]
    Let $(S, d)$ be a Polish space with metric $d: S \times S \gets \mathbb{R}$. Let $P$ and $Q$ be two probability measures on $S$ with finite $p$-moments for a $p\in [1, \infty)$. Then, $p$-Wasserstein distance between $P$ and $Q$ is defined as
    \begin{align*}
        W_p(P,Q) \triangleq \left( \inf_{\gamma \in \Gamma(P,Q)}\int_{S\times S} \|x-y\|^p~\mathrm{d} \Gamma(x,y)\right)^{1/p},
    \end{align*}
    where $\Gamma(P,Q)$ is the set of all couplings of $P$ and $Q$.
\end{definition}
% \twocolumn
\subsection{Properties of the Witness Complexes and Landmarks}\label{app:witness_existing}
We start from presenting the results on the local witness complexes~\citep{arafat2020epsilon} that are integral for our proofs. Specifically,  for a graph $\graph$ Theorem~\ref{thm:enet_results} summarises the size of the landmark set, the stability of the landmark set, and the upper bound on the difference of the persistence diagrams obtained through Vietoris-Rips and local witness complexes.

\begin{theorem}[Properties of Local Witness Complexes \citep{arafat2020epsilon}]\label{thm:enet_results} 
Let $\landmarks$ be the landmark set computed on a graph $\graph$ with $|\nodes|$ nodes and diameter $\mathrm{Diam}(\graph)$.
    \begin{enumerate}
        \item \textbf{Finiteness of the landmark set.} The cardinality of the landmark set $\landmarks$ is $\left(\dfrac{\mathrm{Diam}(\graph)}{\epsilon}\right)^{\bigO\left(\log \frac{|\nodes|}{\epsilon}\right)}$. Here, $\epsilon \triangleq \max_{u, v \in \landmarks} \frac{1}{2} d(u,v)$, and is a tunable parameter.
        \item \textbf{Stability of the landmark set.} {The Hausdorff distance $d_H(\nodes, \landmarks)$ between connected weighted graph $(\nodes, d_{\graph})$ and its $\epsilon$-net induced subspace $(\landmarks, d_{\landmarks})$ is at most $\epsilon$, where $\landmarks \subseteq \nodes$ is the set of landmarks.}
        \item \textbf{3-approximation of Vietoris-Rips.} For any $\alpha > 2\epsilon\geq 0$, where $\epsilon$ corresponds to the $\epsilon$-net induced by the landmark set,
        \begin{align*}
             &\mathrm{VR}_{\alpha/3}(\landmarks) \subseteq \witness_{\alpha}(\nodes,\landmarks) \subseteq \mathrm{VR}_{3\alpha}(\landmarks).
        \end{align*}
        That is, ${W_\infty(\mathrm{PD}_{>2\epsilon}(\mathrm{VR}),\mathrm{PD}_{>2\epsilon}(\witness)) \leq 3\log~3}$, which also implies that
        \begin{align}\label{eq:approx_lw_vr}
             % \implies & \mathrm{PD}_{>2\epsilon}(VR) \text{ and } \mathrm{PD}_{>2\epsilon}(\witness) \text{ are } 3\log 3\text{-approximation of each other.} \\
             & {W_\infty(\mathrm{PD}(\mathrm{VR}(\graph)),\mathrm{PD}(\witness(\graph))) \leq 3\log~3 + 2\epsilon}.
        \end{align}
    \end{enumerate}
\end{theorem}

\subsection{Stability of the Topological Representations}\label{app:analysis}

In this section, we derive stability guarantees of the encoding, computed on a graph $\graph$ by the Witness Graph Topological Layer (WGTL) under a perturbation level $\delta$.
First, Theorem~\ref{thm:stab_local} discusses the stability of the local topology encodings (see Figure~\ref{fig:schematic}). %computed from $\graph$.
Then, Theorem~\ref{thm:stab_global} states the stability of the global topology encodings (see Figure~\ref{fig:wgtl}).
Finally, we leverage these two results to derive the stability of the topological representation learned by WGTL (Figure~\ref{fig:wgtl}).

\begin{remark}
    For building Persistence Images, following~\cite{persistence_images}, we use Gaussian kernel with variance $\sigma$ and a weighting function $w$ with $|\nabla w|=1$. This results in the corresponding
 $\gaussconst \triangleq \left( \sqrt{5}  + \sqrt{\frac{10}{\pi}}\frac{1}{\sigma} \right)$, where $\sqrt{5}$ and $\sqrt{10}$ terms arise due to the relations among norms in $L^2$ and $L^{\infty}$ norms.
 The analysis can be generalized further to other subGaussian kernel, which would lead to modification of the constant $\gaussconst$.
\end{remark}

\medskip
\begin{reptheorem}{thm:stab_local}[Stability of the Local Topology Encoding]
Let us assume that $p < \infty$ and $C_{\epsilon}$ is the maximum size of the $\epsilon$-neighbourhood created by the landmarks. Let us denote the persistence diagram obtained from local topology encoding of $\graph$ by $\mathrm{T}(\graph)$, and the $p$-Wasserstein distance between two probability distributions as $W_p(\cdot, \cdot)$.

\textit{(a)} If \textit{C\v{e}ch complex}, i.e. $\mathrm{Cech}(\mathcal{G})$, is used to compute the local persistence images around each landmark, then for any graph perturbation $\|\mathcal{G}-\mathcal{G}'\|_1 = \bigO(\delta)$ 
\begin{align}
    W_p(\mathrm{T}(\mathcal{G}), \mathrm{T}(\mathcal{G}')) = \bigO( C_{\epsilon} \delta).
\end{align}

\textit{(b)}  If \textit{Witness complex}, i.e. $\mathrm{Wit}(\mathcal{G})$, is used to compute the local persistence images around each landmark, then for any graph perturbation $\|\mathcal{G}-\mathcal{G}'\|_1 = \bigO(\delta)$ %the final persistence diagram representation changes by
\begin{align}
    W_p(\mathrm{T}(\mathcal{G}), \mathrm{T}(\mathcal{G}')) = \bigO( C_{\epsilon} (\delta + \epsilon)).
\end{align}
\end{reptheorem}
\begin{proof} By the property of $L_p$ norms, we have
   \begin{align*}
        W_{p}(\mathrm{T}(\mathcal{G}), \mathrm{T}(\mathcal{G}')) &{\leq} W_{\infty}(\mathrm{T}(\mathcal{G}), \mathrm{T}(\mathcal{G}')), \quad 0<p<\infty.
    \end{align*}

     Let $A(\graph) $ represent the adjacency matrix constructed from the local topology encoding $\mathbf{Z}_{T_L}$ of the nodes, and $\mathrm{T}(\graph)$ is the persistence diagrams of a fixed dimension $d \in \mathbb{Z}_{\geq 0}$ constructed from $A(\graph)$.    
     Now, due to the stability theorem of persistence diagrams~\citep{stability2005}, we obtain
     \begin{align*}
        W_{\infty}(\mathrm{T}(\mathcal{G}), \mathrm{T}(\mathcal{G}'))
        {\leq} \| A(\mathcal{G}) - A(\mathcal{G}') \|_{\infty} {\leq} \| A(\mathcal{G}) - A(\mathcal{G}') \|_{1}
    \end{align*}
    The last inequality holds since $||\cdot||_{\infty}\leq ||\cdot||_1$ for vectors in $\mathbb{R}^n$.

    Thus, we conclude that 
    \begin{align}\label{eq:wass_to_adj}
        W_{p}(\mathrm{T}(\mathcal{G}), \mathrm{T}(\mathcal{G}')) &{\leq} \| A(\mathcal{G}) - A(\mathcal{G}') \|_{1}.
    \end{align}
    
\textit{(a) C\v{e}ch Complex.} Assume that we have access to C\v{e}ch complexes for each of the $\epsilon$-subgraphs, i.e. $\graph^{\epsilon}_l$ computed around landmarks $l \in \landmarks$. 

Now, we observe that the final adjacency matrix has entries in $\{0,1\}$. Thus, the transformation from the local persistence images to the final adjacency matrices computed using Local Topology Encodings is $1$-Lipschitz. Hence,
    \begin{align*}
        % W_{p}(\mathrm{T}(\mathcal{G}), \mathrm{T}(\mathcal{G}')) &\underset{(a)}{\leq} W_{\infty}(\mathrm{T}(\mathcal{G}), \mathrm{T}(\mathcal{G}')) \\
        % &\underset{(b)}{\leq} \| A(\mathcal{G}) - A(\mathcal{G}') \|_{\infty} \\
        % &\underset{(c)}{\leq} \| A(\mathcal{G}) - A(\mathcal{G}') \|_{1} \\
        % &= \| \Phi(\lbrace \mathrm{PI}(\mathcal{G})_1, \ldots, \mathrm{PI}(\mathcal{G})_{|\nodes|}\rbrace) - \Phi(\lbrace \mathrm{PI}(\mathcal{G}')_1, \ldots, \mathrm{PI}(\mathcal{G}')_{|\nodes|}\rbrace) \|_{1} \\
       \| A(\mathcal{G}) - A(\mathcal{G}') \|_{1}
       &{\leq} \sum_{i=1}^{|\nodes|} \| \mathrm{PI}(\mathrm{Cech}(\mathcal{G}_i) - \mathrm{PI}(\mathrm{Cech}(\mathcal{G}'_i)_i \|_{1}\\
        &{=} \sum_{l=1}^{|\landmarks|} |\mathcal{G}^{\epsilon}_l|\| \mathrm{PI}(\mathrm{Cech}(\mathcal{G}^{\epsilon}_l) - \mathrm{PI}(\mathrm{Cech}(\mathcal{G'}^{\epsilon}_l) \|_{1}.%\\
        % &\underset{(f)}{\leq} \gaussconst \sum_{l=1}^{|\landmarks|} |\mathcal{G}^{\epsilon}_l|~\times~W_1(\mathrm{PD}(\mathrm{Cech}(\mathcal{G}^{\epsilon}_l), \mathrm{PD}(\mathrm{Cech}(\mathcal{G'}^{\epsilon}_l))\\
        % &\leq \gaussconst \sum_{l=1}^{|\landmarks|} |\mathcal{G}^{\epsilon}_l|~\times~W_{\infty}(\mathrm{PD}(\mathrm{Cech}(\mathcal{G}^{\epsilon}_l), \mathrm{PD}(\mathrm{Cech}(\mathcal{G'}^{\epsilon}_l))\\
        % &\underset{(g)}{\leq} \gaussconst \sum_{l=1}^{|\landmarks|} |\mathcal{G}^{\epsilon}_l|~\times~\|\mathcal{G}^{\epsilon}_l-\mathcal{G'}^{\epsilon}_l\|_{\infty}\\
        % &\leq \gaussconst \sum_{l=1}^{|\landmarks|} |\mathcal{G}^{\epsilon}_l|~\times~\|\mathcal{G}^{\epsilon}_l-\mathcal{G'}^{\epsilon}_l\|_{1}\\
        % &\underset{(h)}{\leq} \gaussconst \max_l |\mathcal{G}^{\epsilon}_l| ~\|\mathcal{G}-\mathcal{G'}\|_{1}\\
        %  &= \bigO(C_{\epsilon} ~\|\mathcal{G}-\mathcal{G'}\|_{1}).
    \end{align*}
The final equality is due to the fact that
by construction of the local topology encodings,
the persistence images for all nodes $i \in \graph^{\epsilon}_l$ are the same. 

    % Step (a) is true due to the fact that $W_p(x,y) \leq W_q(x,y)$ for $0<p\leq q$ and for all $x,y$.

    % Step (b) is due to the stability theorem of persistence diagrams~\citep{stability2005}. Here, $A(\graph) $ represents the adjacency matrix constructed from the local topology encoding $\mathbf{Z}_{T_L}$ of the nodes, and   $\mathrm{T}(\graph)$ is the persistence diagrams of a fixed dimension $d \in \mathbb{Z}_{\geq 0}$ constructed from $A(\graph)$.

    % Step (c) is true as $l_{\infty}$ norm is less than $l_1$ norm between two vectors.

    % Step (d)

    % Step (d) is true due to $1$-Lipschitzness of the transformation from the local persistence images to the final adjacency matrices computed using Local Topology Encodings. This is true due to the fact that the final adjacency matrix has entries in $\{0,1\}$, and thus, a Lipschitz constant $1$.

    % Equality (e) is due to the fact that the persistence images for all nodes $i \in \graph^{\epsilon}_l$ are the same.

Now, as a direct consequence of the stability theorem of the persistence images~\citep[Theorem 10]{persistence_images}, we obtain
\small{
\begin{align*}
       &~~~~\sum_{l=1}^{|\landmarks|} |\mathcal{G}^{\epsilon}_l|~\times~\| \mathrm{PI}(\mathrm{Cech}(\mathcal{G}^{\epsilon}_l) - \mathrm{PI}(\mathrm{Cech}(\mathcal{G'}^{\epsilon}_l) \|_{1}\\
       &{\leq} \gaussconst \sum_{l=1}^{|\landmarks|} |\mathcal{G}^{\epsilon}_l|~\times~W_1(\mathrm{PD}(\mathrm{Cech}(\mathcal{G}^{\epsilon}_l), \mathrm{PD}(\mathrm{Cech}(\mathcal{G'}^{\epsilon}_l))\\
        &\leq \gaussconst \sum_{l=1}^{|\landmarks|} |\mathcal{G}^{\epsilon}_l|~\times~W_{\infty}(\mathrm{PD}(\mathrm{Cech}(\mathcal{G}^{\epsilon}_l), \mathrm{PD}(\mathrm{Cech}(\mathcal{G'}^{\epsilon}_l)).%\\
        % &\underset{(g)}{\leq} \gaussconst \sum_{l=1}^{|\landmarks|} |\mathcal{G}^{\epsilon}_l|~\times~\|\mathcal{G}^{\epsilon}_l-\mathcal{G'}^{\epsilon}_l\|_{\infty}\\
        % &\leq \gaussconst \sum_{l=1}^{|\landmarks|} |\mathcal{G}^{\epsilon}_l|~\times~\|\mathcal{G}^{\epsilon}_l-\mathcal{G'}^{\epsilon}_l\|_{1}\\
        % &\underset{(h)}{\leq} \gaussconst \max_l |\mathcal{G}^{\epsilon}_l| ~\|\mathcal{G}-\mathcal{G'}\|_{1}\\
        %  &= \bigO(C_{\epsilon} ~\|\mathcal{G}-\mathcal{G'}\|_{1}).
\end{align*}
}
    %Here, for building Persistence Images, we use Gaussian kernel with variance $\sigma$ and a weighting function $w$ with $|\nabla w|=1$. $\gaussconst$ is the constant appearing due to the Gaussian kernel.

Finally, by applying the stability theorem of persistence diagrams~\citep{stability2005} on each of the $\epsilon$-neighborhoods of the landmarks, we find
%\small{
\begin{align*}
       % \sum_{l=1}^{|\landmarks|} |\mathcal{G}^{\epsilon}_l|~\times~\| \mathrm{PI}(\mathrm{Cech}(\mathcal{G}^{\epsilon}_l) - \mathrm{PI}(\mathrm{Cech}(\mathcal{G'}^{\epsilon}_l) \|_{1}
       % &{\leq} \gaussconst \sum_{l=1}^{|\landmarks|} |\mathcal{G}^{\epsilon}_l|~\times~W_1(\mathrm{PD}(\mathrm{Cech}(\mathcal{G}^{\epsilon}_l), \mathrm{PD}(\mathrm{Cech}(\mathcal{G'}^{\epsilon}_l))\\
        &~~~~\sum_{l=1}^{|\landmarks|} |\mathcal{G}^{\epsilon}_l|~\times~W_{\infty}(\mathrm{PD}(\mathrm{Cech}(\mathcal{G}^{\epsilon}_l), \mathrm{PD}(\mathrm{Cech}(\mathcal{G'}^{\epsilon}_l))\\
        &\leq \sum_{l=1}^{|\landmarks|} |\mathcal{G}^{\epsilon}_l|~\times~\|\mathcal{G}^{\epsilon}_l-\mathcal{G'}^{\epsilon}_l\|_{\infty}\\
        &\leq \sum_{l=1}^{|\landmarks|} |\mathcal{G}^{\epsilon}_l|~\times~\|\mathcal{G}^{\epsilon}_l-\mathcal{G'}^{\epsilon}_l\|_{1}\\
        &\leq \max_l |\mathcal{G}^{\epsilon}_l| ~\|\mathcal{G}-\mathcal{G'}\|_{1} = \bigO(C_{\epsilon} ~\|\mathcal{G}-\mathcal{G'}\|_{1}).
\end{align*}
%}
The last two inequalities hold due to the fact that $\|x\|_{\infty} \leq \|x\|_1$ and $\sum_{i=1}^m a_i b_i \leq \left(\max\limits_{i \in \{1,\ldots, m\}} a_i \right)\sum_{i=1}^m b_i$ if $0 \leq a_i, b_i < \infty$, respectively.

Thus, using Equation~\eqref{eq:wass_to_adj}, we conclude that
 \begin{align*}
        W_{p}(\mathrm{T}(\mathcal{G}), \mathrm{T}(\mathcal{G}')) = \bigO(C_{\epsilon} ~\|\mathcal{G}-\mathcal{G'}\|_{1}).
\end{align*}
% \end{proof}

% \begin{repproposition}{thm:stab_global}[Stability of the Final PD in Figure~\ref{fig:schematic} with Witness Complex]
% For any $p < \infty$ and $C_{\epsilon}$ being the maximum size of the $\epsilon$-net created by the landmarks, 
% \end{repproposition}
% \begin{proof}

\textit{(b) Witness Complex.} Now, we consider the scenario where we have access to only the local Witness complexes as it happens in WGTL.
Specifically, we assume that we have access to Witness complexes $\mathrm{Wit}(\graph^{\epsilon}_l)$ for each of the $\epsilon$-subgraphs, i.e. $\graph^{\epsilon}_l$ computed around landmarks $l \in \landmarks$. Then, due to the $1$-Lipschitzness of the mapping from the persistence images of the Witness complex to the adjacency matrix, we get
    \begin{align*}
      \| A(\mathcal{G}) - A(\mathcal{G}') \|_1
        &\leq \sum_{l=1}^{|\landmarks|} |\mathcal{G}^{\epsilon}_l|~\times~\| \mathrm{PI}(\witness(\mathcal{G}^{\epsilon}_l)) - \mathrm{PI}(\witness(\mathcal{G'}^{\epsilon}_l)) \|_{1}.
\end{align*}
Now, similarly applying the relation between the persistence images and persistence diagrams~\citep[Theorem 10]{persistence_images}, we obtain
\begin{align}
        &~~~~\sum_{l=1}^{|\landmarks|} |\mathcal{G}^{\epsilon}_l|~\times~\| \mathrm{PI}(\witness(\mathcal{G}^{\epsilon}_l)) - \mathrm{PI}(\witness(\mathcal{G'}^{\epsilon}_l)) \|_{1}\\
        &\leq \gaussconst \sum_{l=1}^{|\landmarks|} |\mathcal{G}^{\epsilon}_l|~\times~W_1(\mathrm{PD}(\witness(\mathcal{G}^{\epsilon}_l)), \mathrm{PD}(\witness(\mathcal{G'}^{\epsilon}_l))\notag\\
        &\leq \gaussconst \sum_{l=1}^{|\landmarks|} |\mathcal{G}^{\epsilon}_l|~\times~W_{\infty}(\mathrm{PD}(\witness(\mathcal{G}^{\epsilon}_l)), \mathrm{PD}(\witness(\mathcal{G'}^{\epsilon}_l)).\label{eq:pi_pd_witness}
\end{align}
Now, using the relation between the Vietoris-Rips and the Witness complexes (as in Theorem~\ref{thm:enet_results}.3.) as well as the triangle inequality, we observe that
\begin{align*}
    &~~~~W_{\infty}(\mathrm{PD}(\witness(\mathcal{G}^{\epsilon}_l)), \mathrm{PD}(\witness(\mathcal{G'}^{\epsilon}_l))\\
    &\leq 2 W_{\infty}(\mathrm{PD}(\mathrm{\mathrm{VR}}(\mathcal{G}^{\epsilon}_l)), \mathrm{PD}(\mathrm{VR}(\mathcal{G'}^{\epsilon}_l)) + 6 \log 3 + 4\epsilon.
\end{align*}
Further application of the stability theorem of persistence diagrams constructed from Vietoris-Rips complex~\citep{chazal2009proximity} yields
\begin{align}
    &~~~~W_{\infty}(\mathrm{PD}(\witness(\mathcal{G}^{\epsilon}_l)), \mathrm{PD}(\witness(\mathcal{G'}^{\epsilon}_l))\notag \\
    &\leq 2\|\mathcal{G}^{\epsilon}_l-\mathcal{G'}^{\epsilon}_l\|_{\infty} + 6 \log 3 + 4\epsilon\notag\\
    &\leq 2\|\mathcal{G}^{\epsilon}_l-\mathcal{G'}^{\epsilon}_l\|_{1} + 6 \log 3 + 4\epsilon,\label{eq:wit_vr_approx}
\end{align}
where $\epsilon$ corresponds to the $\epsilon$-net induced by the landmark set $\landmarks$ on the graphs $\graph$.

Combining Equations~\eqref{eq:pi_pd_witness} and~\eqref{eq:wit_vr_approx} results in
        \begin{align}
         &~~~~\| A(\mathcal{G}) - A(\mathcal{G}') \|_1 \notag \\
         &\leq \sum_{l=1}^{|\landmarks|} |\mathcal{G}^{\epsilon}_l|~\times~\| \mathrm{PI}(\witness(\mathcal{G}^{\epsilon}_l)) - \mathrm{PI}(\witness(\mathcal{G'}^{\epsilon}_l)) \|_{1} \notag \\
        &\leq \gaussconst \sum_{l=1}^{|\landmarks|} |\mathcal{G}^{\epsilon}_l|~\times~\left( 2\|\mathcal{G}^{\epsilon}_l-\mathcal{G'}^{\epsilon}_l\|_{1} + 6 \log 3 + 4\epsilon\right) \notag         \\
        &\leq \gaussconst \max_l |\mathcal{G}^{\epsilon}_l| ~\left( 2 \|\mathcal{G}-\mathcal{G'}\|_{1} + 6 \log 3 + 4\epsilon\right)\notag\\
        &\leq \gaussconst C_{\epsilon} \max_l |\mathcal{G}^{\epsilon}_l| ~\left( 2 \|\mathcal{G}-\mathcal{G'}\|_{1} + 6 \log 3 + 4\epsilon\right).\label{eq:local_wit_adj}
    \end{align}
Armed the result in Equation~\eqref{eq:wass_to_adj}, we conclude that
\begin{align*}
    W_{p}(\mathrm{T}(\mathcal{G}), \mathrm{T}(\mathcal{G}'))  = \bigO(C_{\epsilon} ~\left(\|\mathcal{G}-\mathcal{G'}\|_{1}+\epsilon\right)).
\end{align*}
\end{proof}

\begin{repproposition}{thm:stab_global}[Stability of the Global Topology Encoding]\,
\textit{(a)} If we have access to the C\v{e}ch simplicial complex $\mathrm{Cech}^{\mathrm{glob}}(\mathcal{G})$ for $\mathcal{G}$, then for any graph perturbation $\|\mathcal{G}-\mathcal{G}'\|_1 = \bigO(\delta)$ the global topology encoding satisfies
\begin{align*}
    \|\mathrm{PI}(\mathrm{Cech}^{\mathrm{glob}}(\mathcal{G})) - \mathrm{PI}(\mathrm{Cech}^{\mathrm{glob}}(\mathcal{G}'))\|_{\infty} = \bigO( \delta).
\end{align*}
\textit{(b)} If the landmarks $\landmarks$ selected to compute the global Witness complex $\mathrm{Wit}^{\mathrm{glob}}(\mathcal{G})$ induce an $\epsilon$-net of the graph $\graph$ with $\epsilon > 0$, then for any graph perturbation $\|\mathcal{G}-\mathcal{G}'\|_1 = \bigO(\delta)$ the global topology encoding satisfies
\begin{align}
    \|\mathrm{PI}(\witness^{\mathrm{glob}}(\mathcal{G})) - \mathrm{PI}(\witness^{\mathrm{glob}}(\mathcal{G}'))\|_{\infty} = \bigO( \delta + \epsilon).
\end{align}

\end{repproposition}
\begin{proof}\,

\textit{Part (a).} Let us first prove the results for the condition where we have access to the C\v{e}ch complex $\mathrm{Cech}^{\mathrm{glob}}(\mathcal{G}))$ on graph $\graph$. We denote the global persistence image computed from the witness complex on $\mathrm{Cech}^{\mathrm{glob}}(\mathcal{G}))$.

Then, 
\begin{align*}
    &~~~~\|\mathrm{PI}(\mathrm{Cech}^{\mathrm{glob}}(\mathcal{G})) - \mathrm{PI}(\mathrm{Cech}^{\mathrm{glob}}(\mathcal{G}'))\|_{\infty}\\
    &\leq \|\mathrm{PI}(\mathrm{Cech}^{\mathrm{glob}}(\mathcal{G})) - \mathrm{PI}(\mathrm{Cech}^{\mathrm{glob}}(\mathcal{G}'))\|_1 \\ 
    &\leq  \gaussconst W_1\left( \mathrm{PD}(\mathrm{Cech}^{\mathrm{glob}}(\mathcal{G})), \mathrm{PD}(\mathrm{Cech}^{\mathrm{glob}}(\mathcal{G}'))\right).
\end{align*}
Here the first inequality is due to the fact that $||\cdot||_{\infty}\leq ||\cdot||_1$ for vectors in $\mathbb{R}^n$. The second inequality is due to the stability theorem of persistence images with Gaussian kernels~\citep[Theorem 10]{persistence_images}.

Now, applying the relation between $W_{\infty}$ and $W_1$ distances, we get
\begin{align*}
&~~~~W_1\left( \mathrm{PD}(\mathrm{Cech}^{\mathrm{glob}}(\mathcal{G})), \mathrm{PD}(\mathrm{Cech}^{\mathrm{glob}}(\mathcal{G}'))\right) \\
&\leq W_{\infty}\left( \mathrm{PD}(\mathrm{Cech}^{\mathrm{glob}}(\mathcal{G})), \mathrm{PD}(\mathrm{Cech}^{\mathrm{glob}}(\mathcal{G}'))\right).%\\
% &{\leq} 2 W_{\infty}\left( \mathrm{PD}(\mathrm{VR}^{\mathrm{glob}}(\mathcal{G})), \mathrm{PD}(\mathrm{VR}^{\mathrm{glob}}(\mathcal{G}'))\right) +6 \log 3 + 4\epsilon
\end{align*}

Now, by applying the stability theorem of persistence diagram of Vietoris-Rips complex~\citep{chazal2009proximity}, we arrive to
\begin{align*}
    &~~~~W_{\infty}\left( \mathrm{PD}(\mathrm{Cech}^{\mathrm{glob}}(\mathcal{G})), \mathrm{PD}(\mathrm{Cech}^{\mathrm{glob}}(\mathcal{G}'))\right)\\
    &\leq \| \mathcal{G} - \mathcal{G}'\|_{\infty} \leq \| \mathcal{G} - \mathcal{G}'\|_{1}= \bigO( \delta).
\end{align*}

\textit{Part b: Witness complex.} Following the similar steps as the derivation for C\v{e}ch complex, we get
\begin{align*}
    &~~~~\|\mathrm{PI}(\witness^{\mathrm{glob}}(\mathcal{G})) - \mathrm{PI}(\witness^{\mathrm{glob}}(\mathcal{G}'))\|_{\infty}\\
    &\leq W_{\infty}\left( \mathrm{PD}(\witness^{\mathrm{glob}}(\mathcal{G})), \mathrm{PD}(\witness^{\mathrm{glob}}(\mathcal{G}'))\right).
\end{align*}
Due to the 3-approximation theorem of Vietoris-Rips complex with Witness complex (Theorem~\ref{thm:enet_results}.3.), we find
\small{
\begin{align*}
&~~~~W_{\infty}\left( \mathrm{PD}(\witness^{\mathrm{glob}}(\mathcal{G})), \mathrm{PD}(\witness^{\mathrm{glob}}(\mathcal{G}'))\right)\\
&\leq 2 W_{\infty}\left( \mathrm{PD}(\mathrm{VR}^{\mathrm{glob}}(\mathcal{G})), \mathrm{PD}(\mathrm{VR}^{\mathrm{glob}}(\mathcal{G}'))\right) +6 \log 3 + 4\epsilon.
\end{align*}}
Now, by applying the stability theorem of persistence diagram of Vietoris-Rips complex~\citep{chazal2009proximity}, we get
\begin{align*}
    &W_{\infty}\left( \mathrm{PD}(\mathrm{VR}^{\mathrm{glob}}(\mathcal{G})), \mathrm{PD}(\mathrm{VR}^{\mathrm{glob}}(\mathcal{G}'))\right)\\
    \leq~~&\| \mathcal{G} - \mathcal{G}'\|_{\infty} \leq \| \mathcal{G} - \mathcal{G}'\|_{1}= \bigO( \delta ).
\end{align*}
Combining these steps, we conclude that
\begin{align*}
    \|\mathrm{PI}(\witness^{\mathrm{glob}}(\mathcal{G})) - \mathrm{PI}(\witness^{\mathrm{glob}}(\mathcal{G}'))\|_{\infty} = \bigO( \delta + \epsilon).
\end{align*}

\end{proof}

% \deb{I do not know about the GCN layer 1 and CNN layer on global PI. So, I just consider them as bounded transformations on the same domain. Need confirmation! Or the other way is that we think of the node features thrown to the attention as three adjacency matrices and work on them!}

\begin{repproposition}{thm:stab_WGTL}[Stability of the attention-driven node representation in Figure~\ref{fig:wgtl}]
If the landmarks selected for the Witness complex induce an $\epsilon$-net of the graph with $\epsilon > 0$, we obtain that for any graph perturbation $\|\mathcal{G}-\mathcal{G}'\|_1 = \bigO(\delta)$ the global topology encoding satisfies
\begin{align}
 \| {\boldsymbol{Z}}_{\mathrm{WGTL}} (\mathcal{G}) - {\boldsymbol{Z}}_{\mathrm{WGTL}} (\mathcal{G}') \|_{1} = \bigO(C_{\epsilon}(C_{\epsilon} + L_{GNN}) (\delta + \epsilon)^2).
     % \| Z_{WGTL} (\mathcal{G}) - Z_{WGTL}  (\mathcal{G}') \|_{1} = \bigO( (C_{\epsilon} + L_{GCN}) \delta + C_{\epsilon} \epsilon).
\end{align}
\end{repproposition}
\begin{proof} We proceed with the proof in three steps.

\textit{Step 1: Decomposition to three components.} We begin the proof by decomposing the WGTL representation of $\graph$ into its three components. Specifically,
\begin{align*}
    {\boldsymbol{Z}}_{\mathrm{WGTL}} (\mathcal{G}) = \left(\alpha_{T_L} \boldsymbol{A}(\mathcal{G})  +  \alpha_G \boldsymbol{A}^{\mathrm{GNN}}(\mathcal{G}) \right) \boldsymbol{A}^{glob}(\mathcal{G}),
\end{align*}
where $\boldsymbol{A}(\graph) $ represents the adjacency matrix constructed from the local topology encoding $\boldsymbol{Z}_{T_L}$ of the nodes, $\boldsymbol{A}^{glob}(\graph) $ represents the adjacency matrix constructed from the global topology encoding $\boldsymbol{Z}_{T_G}$ of the nodes, $\boldsymbol{A}^{\mathrm{GCN}}$ represents the adjacency matrix constructed from the encoding $\boldsymbol{Z}_{G}^{(m+1)}$ of the nodes obtained from GNNs, and $\alpha_{T_L}$ and $\alpha_{G}$ are non-negative attention weights in $(0,1)$ as described in Section~\ref{sec:wgtl}.

\textit{Step 2: Stability of the three individual components.} In order to prove the stability of the WGTL representation, we first prove the stability results of each of the components.
    \begin{enumerate}
    \item  For the local PIs passing through transformer, we have from Equation~\eqref{eq:local_wit_adj}
    {\small
    \begin{align*}
        \| \boldsymbol{A}(\mathcal{G}) - \boldsymbol{A}(\mathcal{G}') \|_{1} 
       &\leq \gaussconst C_{\epsilon} \left( \|\mathcal{G}-\mathcal{G'}\|_{1} + 6 \log 3 + 4\epsilon\right).
    \end{align*}
    }
    The constants follow directly from Theorem~\ref{thm:stab_local}.
    
    \item For the global PIs passing through CNN, we have from Proposition~\ref{thm:stab_global}
    {\small
    \begin{align*}
        \| \boldsymbol{A}^{glob}(\mathcal{G}) - \boldsymbol{A}^{glob}(\mathcal{G}') \|_{\infty} 
        % &\leq \| \boldsymbol{A}^{glob}(\mathcal{G}) - \boldsymbol{A}^{glob}(\mathcal{G}') \|_{1} \\
        % &\leq \|\mathrm{PI}^{\mathrm{glob}}(\mathcal{G}) - \mathrm{PI}^{\mathrm{glob}}(\mathcal{G}')\|_1 \\
        &\leq  \gaussconst \left(\| \mathcal{G} - \mathcal{G}'\|_{1} +6 \log 3 + 4\epsilon\right).
    \end{align*}
    }
    The constants follow the Proposition~\ref{thm:stab_global}.

    \item For the graph passing through GNN, following~\cite{gcnLips}, we obtain
    \begin{align*}
        \| \boldsymbol{A}^{\mathrm{GNN}}(\mathcal{G})  - \boldsymbol{A}^{\mathrm{GNN}}(\mathcal{G'}) \|_1
        &\leq L_{\mathrm{GNN}} \| \mathcal{G} - \mathcal{G'} \|_1.
        % &\leq L_{\mathrm{GNN}} \| \mathcal{G} - \mathcal{G'} \|_1.
    \end{align*}
    % \deb{
    % $Z = D^{-1/2} (A+I) D^{-1/2} X W$  
    % $||| Z - Z' |||_F^2 = ||| D^{-1/2} (A - A') D^{-1/2} X W |||_F^2 \leq \min\{\nodes, d\} ||| D^{-1/2} (A - A') D^{-1/2} X W |||_2^2 \leq d ||| D^{-1}|||_2^2 ||| (A - A') |||_2^2 ||| XW |||_2^2 \leq d \frac{1}{deg_{min}^2} \lambda^2_{max}(XW) ||| (A - A') |||_2^2$.
    % }
    \end{enumerate}
\textit{Step 3: Merging the pieces together.}
{\small
\begin{align*}
    &\| {\boldsymbol{Z}}_{\mathrm{WGTL}} (\mathcal{G}) -{\boldsymbol{Z}}_{\mathrm{WGTL}} (\mathcal{G}') \|_1\\
    = &\| \left(\alpha_{T_L} \boldsymbol{A}(\mathcal{G})  +  \alpha_G \boldsymbol{A}^{\mathrm{GNN}} \right) \boldsymbol{A}^{glob}(\mathcal{G})\\
    &~~~- \left(\alpha_{T_L} \boldsymbol{A}(\mathcal{G}')  +  \alpha_G \boldsymbol{A}^{\mathrm{GNN}}(\mathcal{G}') \right) \boldsymbol{A}^{glob}(\mathcal{G}') \|_1\\
    \underset{(a)}{\leq} &\| \left(\alpha_{T_L} \boldsymbol{A}(\mathcal{G})  +  \alpha_G \boldsymbol{A}^{\mathrm{GNN}} \right) - \left(\alpha_{T_L} \boldsymbol{A}(\mathcal{G}')  +  \alpha_G \boldsymbol{A}^{\mathrm{GNN}}(\mathcal{G}') \right) \|_1\\
    &~~\times~\| \boldsymbol{A}^{glob}(\mathcal{G}) - \boldsymbol{A}^{glob}(\mathcal{G}')\|_{\infty}\\
    \underset{(b)}{\leq} &\left( \|  \boldsymbol{A}(\mathcal{G})  - A(\mathcal{G}') \|_1 +  \|  \boldsymbol{A}^{\mathrm{GNN}} - \boldsymbol{A}^{\mathrm{GNN}}(\mathcal{G}') \|_1 \right)\\
    &~~\times~\| \boldsymbol{A}^{glob}(\mathcal{G}) - \boldsymbol{A}^{glob}(\mathcal{G}')\|_{\infty}\\
    = &\bigO\left(\left(C_{\epsilon}(\delta+\epsilon) + L_{\mathrm{GNN}} \delta\right)C_{\epsilon}(\delta+\epsilon) \right)\\
    = &\bigO\left(\left(C_{\epsilon} + L_{\mathrm{GNN}}\right)C_{\epsilon}(\delta+\epsilon)^2 \right)
\end{align*}
}
Inequality (a) is due to H\"older's inequality. Inequality (b) is due to triangle inequality.

The final result is due to a direct application of the results in Step 2.
   % \begin{align*}
   %     &\| Z_{WGTL} (\mathcal{G}) - Z_{WGTL} (\mathcal{G}') \|_{1}\\ 
   %     \leq &\| \alpha_{loc} A^{loc}(\mathcal{G}) + \alpha_{glob} A^{glob}(\mathcal{G}) + \alpha A^{GCN1}(\mathcal{G}) - \alpha'_{loc} A^{loc}(\mathcal{G'}) + \alpha'_{glob} A^{glob}(\mathcal{G'}) + \alpha' A^{GCN1}(\mathcal{G'}) \|_{1} \\
   %     \leq & \| A^{loc}(\mathcal{G}) - A^{loc}(\mathcal{G'})\|_1 + \| A^{glob}(\mathcal{G}) - A^{glob}(\mathcal{G'})\|_1 + \| A^{GCN1}(\mathcal{G}) - A^{GCN1}(\mathcal{G'})\|_1 \\
   %     \leq & \bigO( (C_{\epsilon} + 1) (\delta + \epsilon) ) + \| A^{GCN1}(\mathcal{G}) - A^{GCN1}(\mathcal{G'})\|_1\\ 
   %     \leq & \bigO( (C_{\epsilon} + L_{GCN}) \delta + C_{\epsilon} \epsilon)
   % \end{align*}
\end{proof}
\subsection{Stability of the Topological Loss}
In this section, we derive the stability bounds on the topological loss (Equation~\eqref{eq:topo_loss}) incurred due to perturbation of the underlying graph $\graph$ by an amount $\delta$. In order to derive the stability result (Proposition~\ref{prop:stab_tloss}), we first bound the topological loss incurred due to a graph $\graph$ and its dependence on the properties of the graph $\graph$ (Theorem~\ref{thm:bound_topoloss}).
\begin{theorem}[Boundedness of Topological Loss]\label{thm:bound_topoloss}
% Let us define
% \begin{align}
%     L_{topo, k}(\mathrm{T}(\mathcal{G})) \triangleq \sum_{i=1}^m (d_i - b_i)^p \left(\frac{d_i+b_i}{2} \right)^{q}.
% \end{align}
Let us assume that the cardinality of the $\epsilon$-neighborhood of any node in $\graph$ grows polynomially, i.e. $C_{\epsilon} = \bigO(\epsilon^{-M})$ for an $M > 0$.
If $m$ is the number of points in the persistence diagram, $k=\max\{p,q\}$ and $2k>M$, $ L_{topo, k}(\mathrm{T}(\mathcal{G}))$ is bounded, such that
{\small
\begin{align}\label{eq:bd_topo_loss}
    L_{topo, k}(\mathrm{T}(\mathcal{G}))\leq  \gaussconst^{2k} {\epsilon}^{-2kM} \mathrm{Diam}(A(\mathcal{G})) + 2^{k-2} m \mathrm{Diam}(\mathcal{G})^{2k}. 
\end{align}
}
\end{theorem}
\begin{proof} We proceed with the proof in three steps.

\textit{Step 1: Decomposition.} First, we decompose the impacts of persistence and birth of complexes on the topological loss by applying a series of algebraic inequalities.
\begin{align}
    &~~~~L_{topo, k}(\mathrm{T}(\mathcal{G})) \\
    &= \sum_{i=1}^m (d_i - b_i)^p \left(\frac{d_i+b_i}{2} \right)^{q} \notag\\
    &\underset{(a)}{\leq} \sum_{i=1}^m \frac{(d_i - b_i)^{2p}}{2} + \frac{1}{2}\left(\frac{d_i+b_i}{2} \right)^{2q}\notag\\
    %&\leq \sum_{i=1}^m \frac{(d_i - b_i)^{2p}}{2} + \frac{1}{2}\left(\frac{d_i+b_i}{2} \right)^{2q}\\
    &\underset{(b)}{\leq} \sum_{i=1}^m \frac{(d_i - b_i)^{2p}}{2} + {2}^{q-1} \left(b_i^{2q} + \left(\frac{d_i-b_i}{2}\right)^{2q} \right)\notag\\
    &\leq \sum_{i=1}^m \left(1+\frac{1}{2^{q}}  \right)\frac{(d_i - b_i)^{\max\lbrace 2p, 2q\rbrace}}{2} + {2}^{q-1} \sum_{i=1}^m b_i^{2q}\notag\\
    &= \frac{1}{2}\left(1+\frac{1}{2^{q}}  \right)\sum_{i=1}^m (d_i - b_i)^{\max\lbrace 2p, 2q\rbrace} + {2}^{q-1} \sum_{i=1}^m b_i^{2q}\notag\\
    &\underset{(c)}{\leq} \sum_{i=1}^m (d_i - b_i)^{2k} + {2}^{q-1} \sum_{i=1}^m b_i^{2q}.\label{eq:bd_loss}
\end{align}
Step (a) is due to the fact that $2xy \leq x^2 + y^2$ for all $x,y \in \mathbb{R}$.

Step (b) holds true as $\left(\frac{x+y}{2}\right)^{2q} = \left(\frac{x-y}{2} + y\right)^{2q} \leq 2^{q-1} \left(\left(\frac{x-y}{2}\right)^{q} + y^{q} \right)^{2} \leq 2^{q} \left( \left(\frac{x-y}{2}\right)^{2q} + y^{2q} \right)$ for $x, y \geq 0$ and $q\geq 1$.

The last inequality (c) holds due to the fact that for any $q\geq 0$, $\frac{1}{2}\left(1+\frac{1}{2^{q}}  \right) \leq 1$.

\textit{Step 2: Controlling impacts of persistence and births.}

Due to \citep{cohen2010lipschitz}, we know that degree $2k$-total persistence for any Lipschitz function $f$ over a triangulable compact metric space ${Dom}$ is upper bounded by $C_{{Dom}} \mathrm{Lip}(f)^{2k}$ for $2k >M$ is bounded by
{\small
\begin{align}
    \sum_{i=1}^m (d_i - b_i)^{2k} \leq C_{A(\mathcal{G})} \mathrm{Lip}(A(\mathcal{G}))^{2k} \leq \gaussconst^{2k} C_{\epsilon}^{2k} \mathrm{Diam}(A(\mathcal{G})).\label{eq:control_persist}
\end{align}
}
The second inequality is due to the Lipschitzness property of persistence images~\citep[Theorem 4]{persistence_images} applied on the local persistence images calculated on an $\epsilon$-neighborhood. Thus, $\mathrm{Lip}(\boldsymbol{A}(\mathcal{G})) \leq \gaussconst C_{\epsilon}$, where $C_{\epsilon}$ is the maximum cardinality of the $\epsilon$-neighborhoods induced by the landmarks.

We also know that the birth and death of topological features on a graph $\graph$ is upper bounded by the diameter of the graph $\mathrm{Diam}(\graph)$, as mentioned in Section~\ref{sec:background}. Thus,
\begin{align}
    \sum_{i=1}^m b_i^{2q} \leq m \mathrm{Diam}(\mathcal{G})^{2q} \leq m \mathrm{Diam}(\mathcal{G})^{2k}.\label{eq:control_birth}
\end{align}

\textit{Step 3: Merging it together.} 
Finally, by applying Equation~\eqref{eq:control_persist} and~\eqref{eq:control_birth} in Equation~\eqref{eq:bd_loss}, we get
{\small
\begin{align*}
    L_{topo, k}(\mathrm{T}(\mathcal{G})) = \\\bigO\left(\gaussconst^{2k} {\epsilon}^{-2kM} \mathrm{Diam}(A(\mathcal{G})) + 2^{k-1} m \mathrm{Diam}(\mathcal{G})^{2k}\right).
\end{align*}
}
\end{proof}

\begin{repproposition}{prop:stab_tloss}[Stability of Topological Loss]
% Let us define
% \begin{align*}
%     L_{topo, k}(\mathrm{T}(\mathcal{G})) \triangleq \sum_{i=1}^m (d_i - b_i)^p \left(\frac{d_i+b_i}{2} \right)^{q}.
% \end{align*}
Let us assume that the cardinality of the $\epsilon$-neighborhood of any node in $\graph$ grows polynomially, i.e. $C_{\epsilon} = \bigO(\epsilon^{-M})$ for an $M > 0$.
If $m$ is the number of points in the persistence diagram, $k=\max\{p,q\}$ and $2k>M$, $ L_{topo, k}(\mathrm{T}(\mathcal{G}))$ is bounded, such that
\small{
\begin{align*}
    &\big| L_{topo, k}(\mathrm{T}(\mathcal{G})) - L_{topo, k}(\mathrm{T}(\mathcal{G'}))\big| =\notag\\
    &\bigO\left(C_{\epsilon}\left(\gaussconst^{2k} {\epsilon}^{-2kM} \mathrm{Diam}(\boldsymbol{A}(\mathcal{G})) + 2^{k-1} m \mathrm{Diam}(\mathcal{G})^{2k}\right) \|\mathcal{G}-\mathcal{G'}\|_{1}\right)
\end{align*}
}
\end{repproposition}

\begin{proof}
By Theorem~\ref{thm:bound_topoloss}, we know that $L_{topo, k}(\mathrm{T}(\mathcal{G}))$ is a bounded function, where the upper bound depends on diameter of the graph $\graph$ and the corresponding local topology encoding $\boldsymbol{A}(\mathcal{G})$. For simplicity, let us denote the bound by $B(\graph)$.

Since every bounded function is Lipschitz with the worst possible Lipschitz constant being the bound itself, we get
\begin{align*}   
&|L_{topo, k}(\mathrm{T}(\mathcal{G})) - L_{topo, k} (\mathrm{T}(\mathcal{G}'))| \\
  %  &= |\sum_{i=1}^m (d_i - b_i)^p \left(\frac{d_i+b_i}{2} \right)^{q} - \sum_{i=1}^m (d'_i - b'_i)^p \left(\frac{d'_i+b'_i}{2} \right)^{q}|\\
\leq~~&B(\graph) \|\boldsymbol{A}(\mathcal{G})-\boldsymbol{A}(\mathcal{G'})\|_{\infty}
    = \bigO\left(B(\graph) C_{\epsilon} \|\graph-\graph'\|_{1}\right).
\end{align*}
The last statement holds true due to Equation~\ref{eq:local_wit_adj}.
    % The second inequality is due to the $\epsilon$-neighbourhood based construction of the topological loss from $\graph$ (refer to Figure~\ref{fig:schematic} and proof of Theorem~\ref{thm:stab_local}).
    We conclude the proof by replacing $B(\graph)$ with the exact expression in Equation~\ref{eq:bd_topo_loss}.
\end{proof}

% \subsection{Stability of Witness PD} \naheed{If landmark $\landmarks$ is an $\epsilon$-net of $\nodes$, $W_\infty(\mathrm{PD}(VR),\mathrm{PD}(\witness)) \leq 3\log~3 + 2\epsilon$. }
% \clearpage
\section{Details of Adversarial Attacks: Configurations and Budgets}\label{app:attacks} There are many adversarial attacks in the literature~\cite{gnnsurvey2022,survey2021,mettack,metapgd,nettack,wu2019adversarial}. In this paper, we validate our method for poisoning-type attacks. Other types of attacks, such as node-injection~\cite{wang2018attack}, reinforcement-learning-based attacks~\cite{dai2018adversarial} and attacks on node-embeddings~\cite{bojchevski2019adversarial} are left as future work.

In this paper, we focus on four different local and global poisoning attacks to evaluate the robustness of our proposed WGTL, and consider a fixed GCN without weight re-training as the surrogate for all attacks. 
All attacks are considered under a non-adaptive setting, meaning it is assumed that the attacker can not adapt or interact with the model during the attack process.
In all the poisoning attacks, we vary the perturbation rate, i.e., the ratio of changed edges, from 0\% to 10\%
% 25\% 
with a step of 5\%. 

\textbf{Global Poisoning Attacks.} Among global (non-targeted) poisoning attacks, we consider mettack~\citep{mettack} and two different topological attacks: PGD~\citep{pgdattack} and its more recent adaptation Meta-PGD~\citep{metapgd}. Mettack treats the graph as a hyperparameter and greedily selects perturbations based on meta-gradient for node pairs until the budget is exhausted. PGD attack~\citep{pgdattack} uses the Projected Gradient Descent algorithm with the constraint $\|\boldsymbol{S}\|_0 \leq \delta$ to minimise attacker loss. Here, $\boldsymbol{S}$ is a binary symmetric matrix with $\boldsymbol{S}_{ij} = 1$ if the $(i,j)$-th entry of the Adjacency matrix is flipped by the attacker, and 0 otherwise. Recently,~\cite{metapgd} proposes to apply PGD on the meta-gradients to design attacks stronger than the greedy mettack. Meta-PGD has been shown to be more effective than mettack in many cases~\citep{metapgd}. Hence, we consider it a more challenging poisoning attack for the proposed method. 

To perform mettack, we keep all the default parameter settings (e.g., $\lambda=0$) following the original implementation~\citep{mettack}. For Cora-ML, Citeseer and Polblogs, we apply the Meta-Self variant of mettack since it is the most effective mettack variant, while for Pubmed, the approximate version of Meta-Self, A-Meta-Self, is applied to save memory and time~\citep{prognn}. We perform the PGD attack with the CE-type attacker loss function. Following the implementation~\citep{pgdattack}, we keep their default parameter settings, i.e., the number of iterations $T = 200$ and learning rate $\eta_t = 200/\sqrt{t}$. For Meta-PGD, we keep the same parameter settings as~\cite{metapgd}, i.e., a learning rate of 0.01 and gradient clipping threshold of 1.

\textbf{Local Poisoning Attack.} Among local attacks, we use nettack~\citep{nettack}. Nettack is a targeted attack which first selects possible perturbation candidates not violating degree distribution and feature co-occurrence of the original graph. Then, it greedily selects, until the attacker's budget is exhausted, the perturbation with the largest score to modify the graph. The score function is the difference in the log probabilities of a target node. 

Following~\citet{nettack}, we vary the number of perturbations made on every targeted node from 1 to 2
% 5
with a step size of 1. Following~\cite{prognn}, the nodes in the test set with a degree $>10$ are treated as target nodes. We only sample 10\% of them to reduce the running time of nettack on Pubmed, while for other datasets, we use all the target nodes.
% \clearpage
\section{Experimental Details}\label{app:details}

\noindent\textbf{Experimental Setup.} All experiments are run on a server with 32 Intel(R) Xeon(R) Silver 4110 CPU @ 2.10GHz processors, 256 GB RAM, and an NVIDIA GPU card with 24GB GPU memory. All models are trained on a single GPU. 
% The source code is available at~\underline{\url{https://github.com/toggled/WGTL}}.
% at~\url{https://www.dropbox.com/scl/fo/0oavxaw0vz2fjdtg1j76c/h?rlkey=b524wsqs60eci9zbryk91rj4a&dl=0}.

\noindent\textbf{Datasets.} Following~\citet{nettack,mettack,prognn}, we validate the proposed approach on 6 benchmark datasets: 5 homophilic and 1 heterophilic graph. The homophilic graph includes three citation graphs: Citeseer, Cora, and Pubmed, one blog graph: Polblogs, and one large-scale graph from Open-Graph Benchmark (OGB): OGBN-arXiv~\cite{ogbn}. As heterophilic graph dataset, we consider Snap-patents~\cite{snap}. For each graph, we randomly choose
10\% of nodes for training, 10\% of nodes for validation, and the remaining 80\% of nodes for testing. For each experiment, we report
the average accuracy of 10 runs. Note that in the
Polblogs graph node features are not available. Following~\citet{prognn}, we set the attribute matrix to $N \times N$ identity matrix.  
\setlength\tabcolsep{1pt}
\begin{table}[t]
\small
\centering
\begin{tabular}{ccccc}
\toprule
\textbf{Dataset} & \textbf{\begin{tabular}[c]{@{}c@{}}\#nodes (LCC)\end{tabular}} & \textbf{\begin{tabular}[c]{@{}c@{}}\#edges (LCC)\end{tabular}} & \textbf{\#classes} & \textbf{\#features} \\ \hline
Cora-ML          & 2,485                                                            & 5,069                                                            & 7                  & 1,433               \\
Citeseer         & 2,110                                                            & 3,668                                                            & 6                  & 3,703               \\
Pubmed           & 19,717                                                           & 44,338                                                           & 3                  & 500                 \\
Polblogs         & 1,222                                                            & 16,714                                                           & 2                  & None                \\ 
OGBN-arXiv & 169,343 & 1,157,799                                                          & 40                  & 128                \\ 
Snap-patents & 4562 & 12103 & 5 & 269 \\
\bottomrule
\end{tabular}
\caption{Dataset statistics: only the largest connected component (LCC) is considered.}
\label{tab:dataset}
\end{table}

% \textbf{Landmark Selection for Local and Global Topology Encodings.} There are several approaches to selecting landmarks, e.g., random selection~\citep{witness}, max-min selection~\citep{witness}, $\epsilon$-net~\citep{arafat2020epsilon} based and centrality-based selection~\citep{topopooling}. In our experiments, we select landmarks based on degree centrality. As shown by~\cite{topopooling}, doing so helps to improve the classification performance. On Cora-ML, Citeseer and Polblogs, we select 5\% nodes, while on Pubmed, we select 2\% nodes as landmarks. Each landmark creates its own cover consisting of a subset of nodes. A node $u$ belongs to the cover of a landmark $l$ if $l$ is nearest to $u$ among all the landmarks. 
% % The average diameters of the cover subgraphs, i.e., average $\epsilon$, take values 1.28, 1.65, 1.10 and 1.7 for Cora-ML, Citeseer, Polblogs and Pubmed graphs, respectively. 
% Due to such landmark selection, the maximum values of $\epsilon$ are 2 for Citeseer and 3 for Cora-ML, Polblogs and Pubmed, and thus ensuring good stability of the encoded local and global topology. %(Corollary 1, Proposition 1).
% % C_{\epsilon}= 262, 95, 153 and 601 for Cora-ML, Citeseer, Polblogs and Pubmed. 
% \clearpage
\section{Ablation studies}\label{app:ablation}
To evaluate the contributions of the different components in our WTGL, we perform ablation studies on Cora-ML and Polblogs datasets under a global attack, i.e., mettack, and a local attack, i.e., nettack. We use GCN as the backbone architecture and consider three ablated variants: (i) GCN+Local Topology Encoding (LTE), (ii) GCN+Global Topology Encoding (GTE), and (iii) GCN+LTE+GTE+Topological Loss (TopoLoss) (i.e., GCN + WTGL). 

The experimental results for mettack and nettack are shown in Table~\ref{abl_mettatack_res_cora_ml} and Table~\ref{abl_nettack_res_cora_ml} respectively. Consistent improvement from the backbone GCN while using GTE, LTE, and topological loss together suggest their importance in an individual as well as in an aggregated manner. %As Tables~\ref{abl_mettatack_res_cora_ml} and~\ref{abl_mettatack_res_polblogs} suggest, under mettack, we observe that all three components contribute to the success of WGTL.

\begin{table}[h!]
\small
\centering
\resizebox{\columnwidth}{!}{
\begin{tabular}{llcccccc}
\toprule
\multirow{2}{*}{\textbf{Dataset}}& \multirow{2}{*}{\textbf{Model}}& \multicolumn{3}{c}{\textbf{Perturbation rate}}
\\
\cmidrule(lr){3-5}%\cmidrule(lr){8-13}
&          & 0\% & 5\% & 10\% \\
           %& 0\% & 5\% & 10\% & 15\% & 20\% & 25\%\\
\midrule
\multirow{4}{*}{{Cora-ML}}& GCN &82.87$\pm$0.83 &76.55$\pm$0.79 &70.39$\pm$1.28 \\
&GCN + LTE & 83.26$\pm$0.43 &77.35$\pm$0.38 &71.27$\pm$0.81 \\
&GCN + GTE &83.37$\pm$1.12 &77.78$\pm$0.59 &70.66$\pm$1.76 \\
&GCN + LTE + GTE + TopoLoss & {\bf 83.83$\pm$0.55} & {\bf 78.63$\pm$0.76} &{\bf 73.41$\pm$0.82} \\
\midrule
\multirow{4}{*}{{Polblogs}}& GCN & 94.40$\pm$1.47 & 71.41$\pm$2.42 & 69.16$\pm$1.86 \\
&GCN + LTE  & {95.34$\pm$0.73} &72.27$\pm$1.07 & 72.02$\pm$0.97 \\
&GCN + GTE & 95.07$\pm$0.09 & 72.78$\pm$0.57 & 73.14$\pm$1.59 \\
&GCN + LTE + GTE + TopoLoss & {\bf 95.95$\pm$0.15} & {\bf72.84$\pm$0.86} & {\bf74.62$\pm$0.42} \\
\bottomrule
\end{tabular}
}
\caption{Ablation studies (Accuracy$\pm$Std) on datasets under Mettack.\label{abl_mettatack_res_cora_ml}}
\end{table}

\begin{table}[h!]
\small
\centering
\resizebox{\columnwidth}{!}{
\begin{tabular}{llcccccc}
\toprule
\multirow{2}{*}{\textbf{Datasets}}& \multirow{2}{*}{\textbf{Model}}& \multicolumn{3}{c}{\textbf{Number of perturbations per node}}
\\
\cmidrule(lr){3-5}%\cmidrule(lr){8-13}
      &    & 0 & 1 & 2 \\
\midrule
\multirow{4}{*}{{Cora-ML}}& GCN &82.87$\pm$0.93 & 82.53$\pm$1.06 & 82.08$\pm$0.81  \\
& GCN + LTE & 82.88$\pm$0.24 & 82.56
$\pm$0.13
 & 82.33$\pm$0.33 \\
& GCN + GTE & 83.15$\pm$0.43 & 82.42$\pm$0.66 & 82.29
$\pm$0.48  \\ %lr=0.0018
% GCN + LTE + GTE + TopoLoss & {\bf 83.83$\pm$0.55} & {\bf 83.41$\pm$0.87} & {\bf 82.74$\pm$0.65} & {\bf 82.10$\pm$0.82} & {\bf \naheed{81.64}$\pm$0.55} &  {\bf 80.98$\pm$0.67} \\
& GCN + LTE + GTE + TopoLoss & {\bf 83.83$\pm$0.55} & {\bf 83.41$\pm$0.87} & {\bf 82.74$\pm$0.65}  \\
\midrule
\multirow{4}{*}{{Polblogs}} & GCN & 94.40$\pm$1.48 &88.91$\pm$1.06 & 85.39$\pm$0.86  \\
& GCN + LTE  & 95.42$\pm$0.58 & 91.45$\pm$0.56 & 88.40$\pm$0.94 \\
& GCN + GTE &  95.07$\pm$0.11 & 91.47
$\pm$0.68 & 89.10$\pm$0.70 \\
& GCN + LTE + GTE + TopoLoss &{\bf 95.95$\pm$0.15} & {\bf 91.47$\pm$0.33} & {\bf 89.10$\pm$0.69} \\
\bottomrule
\end{tabular}
}
\caption{Performance (Accuracy$\pm$Std) on Cora-ML under Nettack.\label{abl_nettack_res_cora_ml}}
\end{table}

\section{Impact of landmark selection algorithm} 
The pseudocode for selecting landmarks for computing global witness topological features and local witness topological features is presented in Algorithm~\ref{alg:landmark}. In order to compute global witness features, we select a set of \emph{global landmark} nodes. In order to compute local witness features, we select a set of \emph{local landmark} nodes for each node in the graph.
\begin{algorithm}[h!]
\caption{{\bf Greedy} Landmark selection algorithm}
  \label{alg:landmark}
  \small
     \textbf{Input:} Graph $\mathcal{G}=(\nodes,\edges)$, $\%$ of nodes as landmarks $p \in (0,1)$\\
    \textbf{Output:} Global landmark set $\landmarks_g$ and Local landmark set $\landmarks$ 
  \begin{algorithmic}[1]
    \STATE Number of landmarks $n_g \gets |\nodes| \cdot p$
    \STATE Sort $\nodes$ in decreasing order of node degrees.
    \STATE Select Global Landmarks $\landmarks_g \gets \nodes[1,2,\ldots,n_g]$
    \FORALL{$l \in \landmarks_g$} 
    \STATE Compute cover $C_l \gets \{u \in \nodes: d_{\graph}(u,l) < d_{\graph}(u,l')\quad\forall l' \in \landmarks_g\setminus\{l\} \}$
    \STATE Compute Subgraph $G_l \gets \mathcal{G}[C_l]$
    \STATE Number of local landmarks $n_l \gets |C_l|\cdot p$
    \STATE Sort $C_l$ in decreasing order of node degrees in $G_l$.
    \STATE Select Local landmarks $\landmarks[l] \gets C_l[1,\ldots,n_l]$
    \ENDFOR
    \RETURN $\landmarks_g, \landmarks$
\end{algorithmic}
\end{algorithm}

In Lines~1-3, we select the set of global landmarks $\landmarks_g$ in order to construct Global witness filtration. We select the top-most $p\%$ highest degree nodes in the graph $\mathcal{G}$ as landmarks.

In Lines~4-9, we select local landmarks corresponding to each global landmark in order to compute the topological features local to each global landmark. A node $u$ that is not a global landmark must be in the cover of some landmark node $l \in \landmarks_g$. We say $u$ is a witness node to the node $l$. We assume the local topological signature does not change inside a cover. In other words, a witness node has the same topological signature as its associated landmark. That is why, instead of computing local landmarks for every node in $\nodes$, we compute only for the global landmarks $\landmarks_g \subseteq \nodes$ (line~4). For each global landmark $l \in \landmarks_g$, we construct its cover $C_l$ in line~5 consisting of all its witness nodes. In line~6, we construct the subgraph $\mathcal{G}[C_l]$ induced by the witness nodes. Finally, in lines~7-9, we select the topmost $p\%$ of the witness nodes with the highest degrees in the induced subgraph $\mathcal{G}[C_l]$ as the local landmark for $l \in \landmarks_g$.

% \subsection{Robustness gain w.r.t with Vietoris-Rips and the trade-offs.} Table~\ref{mettatack_vr}
% {\textcolor{red}{TO DO: Represent table 11 and table 2 in a plot with two different dashed lines for better comparison (for each dataset). Also, try to plot an efficiency vs accuracy trade-off curve and see if such viz. makes more sense.}}

\subsection{Impact of the Number of Landmarks} 
It is well-known that the quality of the Witness complex-based topological features is dependent on the number of landmarks~\citep{witness,arafat2019topological}. Hence, the performance of the proposed witness topological encodings, topological loss, and, finally, the downstream classification quality is also dependent on the number of landmarks. 

In order to study the accuracy and efficiency of WGTL under different numbers of landmarks, we use Algorithm~\ref{alg:landmark} to select ${1\%,10\%, 50\%}$ nodes as landmarks, and in Figure~\ref{fig:ablation_land}, we present the accuracy and computation time of the local and global witness complex-based features on Cora-ML and Citeseer datasets. We observe that increasing the number of landmarks indeed slightly increases the accuracy, albeit with the expense of increased computation time. Due to this trade-off between accuracy and efficiency, the selection of an optimal number of landmarks is dependent on how much robustness is desired by a user within a given computation-time budget. 
\begin{figure}[t]
    \centering
    \subfigure[]
	{
        \includegraphics[width=0.7\columnwidth]
        % [scale=0.3]
        {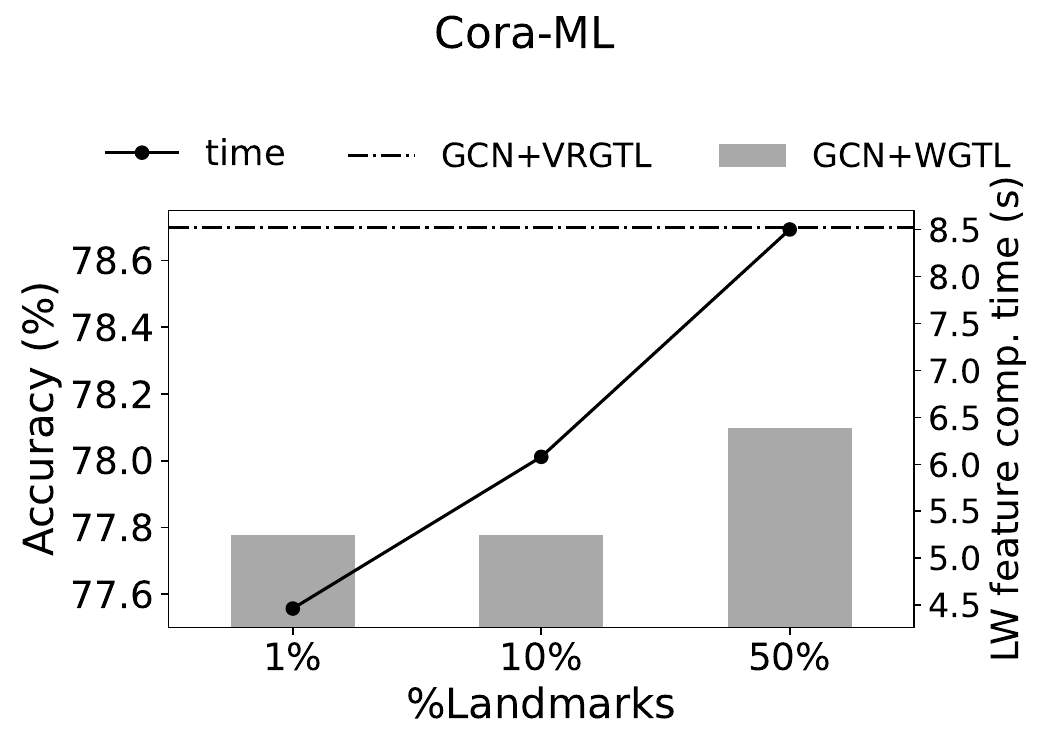}\label{fig:cora_land}
    }
    \subfigure[]
	{
        \includegraphics[width=0.7\columnwidth]
        % [scale=0.3]
        {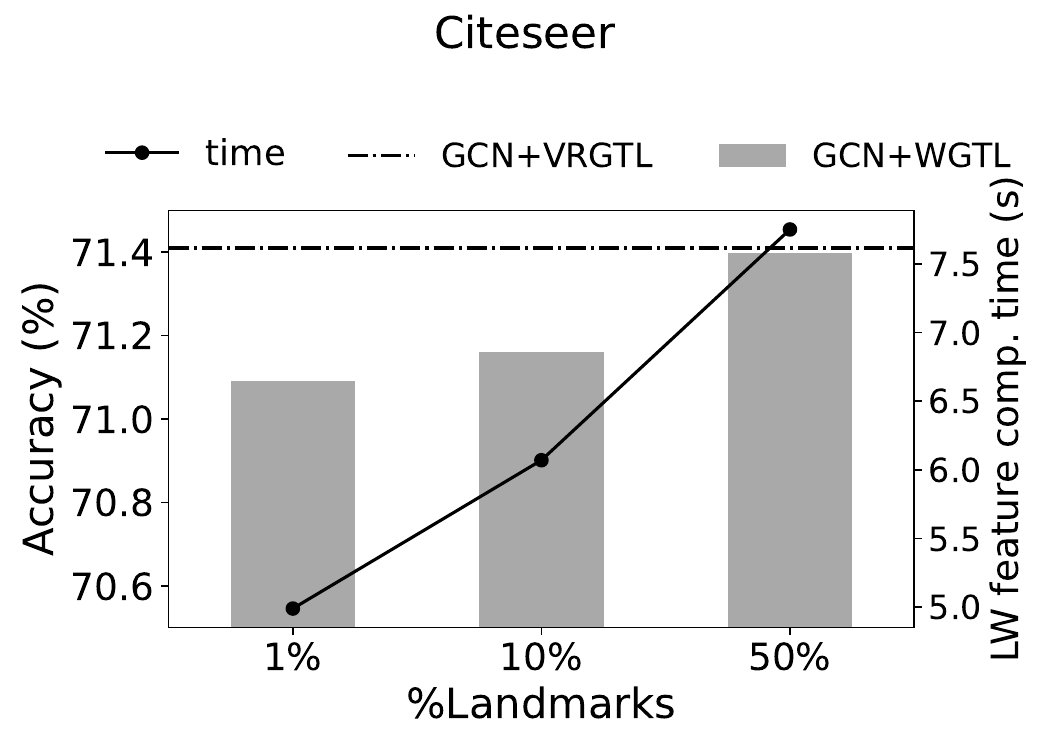}\label{fig:cite_land}
    }
    % \vspace{-.5mm}
    \caption{The trade-off between accuracy and Feature computation time of GCN+WGTL with different numbers of landmarks. The figures are under mettack with $5\%$ perturbation rate.}\label{fig:ablation_land}%\vspace{-.8em}
\end{figure}
\begin{table}[t]
    \centering
    % \begin{tabular}{cccc}
    % \toprule
    %       & Cora-ML & Polblogs & Citeseer \\ \cmidrule(lr){2-4}
    %     VR feature comp. time (seconds) & 59.35&  33.52& 18.73 \\ \hline
    %     Witness feature comp. time (seconds) & 0.03 & 0.05 & 0.01\\
    %     \bottomrule
    % \end{tabular}
    \resizebox{\columnwidth}{!}{
    \begin{tabular}{cccc}
    \toprule
         & Method & Cora-ML &  Citeseer \\ 
         % \cmidrule(lr){3-4}
        \midrule
        \multirow{2}{*}{VR feature comp. time (s)} & CPU (Ripser) &  
        23286.0 & 1698.0 \\ 
         & GPU (Ripser++) &  
        13021.5 & 1430.6 \\ 
\midrule
        Witness feature comp. time (s) & Algorithm 1 & 8.5  & 7.8\\
        \bottomrule
    \end{tabular}
    }
    \caption{Execution times (in seconds) for computing global and local topological features. Landmark selection time is included.}
    \label{tab:tm_global}
\end{table}
\subsection{Comparison of WGTL and Vietoris-Rips based Topology Encoding (VRGTL)}
Finally, we also observe that, on these datasets, the accuracy achieved by WGTL with 50\% landmarks is close to the accuracy achieved by adopting Vietoris-Rips-based topological feature encoding, as indicated by the dotted line representing GCN+VRGTL. A more in-depth comparison among GCN, GCN+WGTL and GCN+VRGTL is presented in Figure~\ref{fig:ablation_land} where we compare their accuracy on Cora-ML and Citeseer under mettack. We observe that the accuracy of GCN+VRGTL is comparable to that of GCN+WGTL. These observations also highlight the flexibility of WGTL in adopting other approximate topological features. 

However, computing Vietoris-Rips features is significantly more expensive than witness topological features~\citep{arafat2020epsilon}. We computed Vietoris-Rips features using both CPU and GPU based state-of-the-art implementations: Ripser~\cite{ripser} and Ripser++~\cite{ripser++}. As shown in Table~\ref{tab:tm_global}, adopting CPU-based Vietoris-Rips PH computation into our computation pipeline incurs more than 2500x (200x) more computation time on Cora-ML (Citeseer) compared to using Witness PH computation. GPU acceleration using Ripser++ is more computationally viable than CPU-based implementation yet 1500x and 180x slower than Witness feature computation on Cora-ML and Citeseer respectively. 

\textit{The results demonstrate that instead of incurring 2 to 3 order of magnitude less computational time, deploying  WGTL leads to similar or better accuracy to those of VRGTL across a wide range of perturbations.}
\section{Experimental Results: WGTL against Different Attacks and Graph Structures}
\label{app:moreresults}
\subsection{Performance of WGTL on targetted attack: Nettack}
\begin{table}[!h]
\setlength\tabcolsep{1pt}
\small
\centering
\setlength\tabcolsep{1pt}
\resizebox{\columnwidth}{!}{
\begin{tabular}{@{}cccccc@{}}
\toprule
\multirow{2}{*}{\textbf{Dataset}}&\multirow{2}{*}{\textbf{Model}}& \multicolumn{3}{c}{\textbf{Number of perturbations per node}}
\\
\cmidrule(lr){3-5}%\cmidrule(lr){8-13}%76.38$\pm$1.40
           & & 0 & 1 & 2\\
\midrule
\multirow{4}{*}{{Cora-ML}} 
& GCN + GNNGuard & 83.21 $\pm$0.34& 82.81$\pm$0.43& 82.51$\pm$0.26 \\
& GCN + GNNGuard + WGTL (ours) & {\bf 84.78$\pm$0.43} &{\bf 84.25$\pm$0.73} & {\bf 83.74$\pm$0.96}  \\
\cmidrule(lr){2-5}
& SimP-GCN &79.52$\pm$1.81 & 74.75$\pm$1.40 & 70.87$\pm$1.70 \\
& SimP-GCN + WGTL (ours) &  \bf{81.49$\pm$0.52} &  \bf{ 76.65$\pm$0.65} &  \bf{72.88$\pm$0.83} \\
\hline
\multirow{4}{*}{{Citeseer}}& GCN + GNNGuard & 71.82$\pm$0.43 & 71.87$\pm$0.53 & 71.71$\pm$0.32 \\
&GCN + GNNGuard + WGTL (ours) & {\bf 72.53$\pm$0.35} & {\bf 72.43$\pm$0.48} & {\bf 71.95$\pm$0.88} \\
\cmidrule(lr){2-5}
& SimP-GCN & 73.73$\pm$1.54 & 73.06$\pm$2.09 & 72.51$\pm$1.25  \\
& SimP-GCN + WGTL (ours) &  \bf{74.32$\pm$0.19
} &  \bf{74.05$\pm$0.71} &  \bf{73.09$\pm$0.50} \\
\hline
% \multirow{2}{*}{{Pubmed}}& GNNGuard & 83.63& 83.62& 83.54& 83.49& 83.40& 83.22 \\
% &GNNGuard + WGTL (ours) &  &  &  & &  & \\
% \hline
\multirow{4}{*}{{Polblogs}}
& GCN + GNNGuard & 95.03$\pm$0.25 & 91.43$\pm$0.36 & 89.45$\pm$0.46 \\
& GCN + GNNGuard + WGTL (ours) & {\bf 96.22$\pm$0.25}& {\bf 91.89$\pm$0.57} & {\bf 90.35$\pm$0.81}  \\
\cmidrule(lr){2-5}
& SimP-GCN & 89.78$\pm$6.47 & 65.75$\pm$5.03 & 61.53$\pm$6.41 \\
&SimP-GCN+WGTL (ours) & \bf{94.56$\pm$0.24} & \bf{69.78$\pm$4.10} & \bf{69.55$\pm$4.42}  \\
\bottomrule
\end{tabular}
}
\caption{\label{SOTA_nettack}Performances  wrt existing defenses under nettack}%Random Attack
% \vspace*{-1.5em}
\end{table}
\begin{table}[!h]
\setlength\tabcolsep{1pt}
\small
\centering
\setlength\tabcolsep{1pt}
\resizebox{\columnwidth}{!}{
\begin{tabular}{@{}ccccc@{}}
\toprule
\textbf{Dataset}&\textbf{Model}& \multicolumn{3}{c}{\textbf{Number of perturbations per node}}
\\ \midrule
% \cmidrule(lr){3-5}%\cmidrule(lr){8-13}
           & & 0 & 1 & 2 \\
\midrule
\multirow{10}{*}{{Cora-ML}} & GCN & 82.87$\pm$0.93 & 82.53$\pm$1.06 & 82.08$\pm$0.81  \\
& GCN + WGTL (ours) & {\bf 83.83$\pm$0.55} & $^*${\bf 83.41$\pm$0.87} & {\bf 82.74$\pm$0.65} 
\\
\cmidrule(lr){2-5} & Chebnet & 80.74$\pm$0.42  & 79.36$\pm$0.67& 77.89$\pm$0.46 \\
& Chebnet+WGTL (ours) & {\bf 82.96$\pm$1.08}  &{\bf 82.90 $\pm$1.14} & {\bf 82.47 $\pm$ 0.96}\\
\cmidrule(lr){2-5} & GAT & 84.25$\pm$0.67	 & 79.88$\pm$1.09 & 72.63$\pm$1.56	\\
& GAT + WGTL (ours) & $^*${\bf 86.07$\pm$2.10} & {\bf 80.80$\pm$0.87} & {\bf 75.80$\pm$0.79} \\
\cmidrule(lr){2-5}
& GraphSAGE & 81.01$\pm$0.27 & 80.48$\pm$0.71 &	80.19$\pm$0.49\\
& GraphSAGE + WGTL (ours) & {\bf 83.63$\pm$0.35} & {\bf 83.23$\pm$0.21} &  $^*${\bf 82.79$\pm$0.36} \\
\hline 
\multirow{8}{*}{{Polblogs}}& GCN & 94.40$\pm$1.48 &88.91$\pm$1.06 & 85.39$\pm$0.86 \\
& GCN + WGTL (ours) & 
$^*${\bf95.95$\pm$0.15} & $^*${\bf 91.47$\pm$0.33} & {\bf 89.10$\pm$0.69}  \\
\cmidrule(lr){2-5} & Chebnet &73.10$\pm$7.13 &65.92$\pm$5.77 & 63.19$\pm$ 5.02\\
& Chebnet+WGTL (ours) & {\bf 92.50$\pm$1.10} & {\bf 89.33$\pm$0.62} & {\bf 88.24$\pm$0.64} \\
\cmidrule(lr){2-5}
& GAT & 95.28$\pm$0.51 & 89.86$\pm$0.63 &	86.44$\pm$1.47  \\
&GAT + WGTL (ours) &{\bf 95.87$\pm$0.26} & {\bf 90.69$\pm$0.51} & {\bf 87.73$\pm$0.38} \\
\cmidrule(lr){2-5}
&GraphSAGE & 94.54$\pm$0.27 & 90.20$\pm$0.30 & 89.57$\pm$0.62 \\
&GraphSAGE + WGTL (ours) & {\bf 95.58$\pm$0.50} & {\bf 90.98$\pm$0.27} & $^*${\bf 89.95$\pm$0.78} \\
\bottomrule
\end{tabular}
}
\caption{Performance wrt GNN backbones (avg. accuracy$\pm$std.) under nettack}
\label{nettack_results2}
\end{table}
\subsection{Performance of WGTL with SGC backbone}
In the main paper, we have deployed WGTL with four GNN backbones: vanilla GCN, GAT, GraphSAGE and ChebNet. In order to test the versatility and flexibility of WGTL, we adopt WGTL with a more recent GNN, namely SGC~\cite{sgc2019} as backbone. We adopt the experimental setup and landmark selection scheme described earlier in Appendix~\ref{app:details}. For the attacks, we adopt the configurations and budgets as described earlier in Appendix~\ref{app:attacks}. Table~\ref{meta_sgc} and table~\ref{net_sgc} demonstrate the performance of the backbones with and without WGTL on three representative datasets: Cora-ML, Citeseer and Polblogs, and under two different representative attacks, i.e., mettack and nettack. \textit{We observe that incorporating WGTL improve the corresponding performance across a range of perturbation rates for both attacks.}
\begin{table}[t]
\small
\setlength\tabcolsep{2pt}
\centering
\resizebox{0.98\columnwidth}{!}{
\begin{tabular}{llcccccc}
\toprule
\multirow{2}{*}{\textbf{Dataset}}&\multirow{2}{*}{\textbf{Model}}& \multicolumn{3}{c}{\textbf{Perturbation Rate}}
\\
\cmidrule(lr){3-5}%\cmidrule(lr){8-13}%76.38$\pm$1.40
           & & 0\% & 5\% & 10\% \\
\midrule
\multirow{2}{*}{{Cora-ML}} & SGC & 82.51$\pm$0.11 & 76.45$\pm$0.11 &	68.74$\pm$0.13 \\
& SGC + WGTL (ours) & {\bf 84.25$\pm$0.06} & {\bf 77.40$\pm$0.73} & {\bf 69.07$\pm$0.46}  \\
\hline
\multirow{2}{*}{{Citeseer}}& SGC & \bf 71.68 $\pm$0.11 & 68.14$\pm$0.17 & 63.12$\pm$0.16 \\

&SGC + WGTL (ours) & 71.33$\pm$0.20 & {\bf 70.15$\pm$0.31} & {\bf 66.18$\pm$0.32}  \\
\hline
% \multirow{2}{*}{{Pubmed}}& GraphSAGE & 83.32	& 78.99 &	76.33 &73.40 & 71.18 & 68.10 \\
% &GraphSAGE + WGTL (ours) &  &  &  & &  & \\
% \hline
\multirow{2}{*}{{Polblogs}}& SGC & 90.08$\pm$0.26&	68.88 $\pm$0.29 &	65.73$\pm$0.15  \\
&SGC + WGTL (ours) & {\bf 95.52$\pm$0.05} & {\bf 74.40$\pm$0.09} & {\bf 70.35$\pm$0.63}  \\
\bottomrule
\end{tabular}
}
\caption{Performance (Accuracy$\pm$Std) of SGC and SGC+WGTL
under mettack.\label{meta_sgc}}%Random Attack
% % \vspace*{-1.5em}
\end{table}
\begin{table}[t]
\small
% \setlength\tabcolsep{2pt}
% \centering
\setlength\tabcolsep{1pt}
\resizebox{\columnwidth}{!}{
\begin{tabular}{llcccccc}
\toprule
\multirow{2}{*}{\textbf{Dataset}}&\multirow{2}{*}{\textbf{Model}}& \multicolumn{3}{c}{\textbf{Number of perturbations per node}}
\\
\cmidrule(lr){3-5}%\cmidrule(lr){8-13}%76.38$\pm$1.40
           & & 0 & 1 & 2 \\
\midrule
\multirow{2}{*}{{Cora-ML}} & SGC & 82.51$\pm$0.11 & 81.91$\pm$0.15 & 80.93$\pm$0.17 \\
& SGC + WGTL (ours) & {\bf 84.25$\pm$0.06} & {\bf 83.85$\pm$0.12} & {\bf 83.35$\pm$0.16}  \\
\hline
\multirow{2}{*}{{Citeseer}}& SGC & \bf{71.68$\pm$0.11} & 71.42 $\pm$0.03 & 70.82 $\pm$0.11 \\
&SGC + WGTL (ours) & 71.33$\pm$0.20 & {\bf 71.45$\pm$0.29} & {\bf 71.16$\pm$0.37}  \\
\hline
% \multirow{2}{*}{{Pubmed}}& GraphSAGE & 83.32	& 78.99 &	76.33 &73.40 & 71.18 & 68.10 \\
% &GraphSAGE + WGTL (ours) &  &  &  & &  & \\
% \hline
\multirow{2}{*}{{Polblogs}}& SGC & 90.08$\pm$0.26&	82.62 $\pm$0.42 &	78.69$\pm$0.33  \\
&SGC + WGTL (ours) & {\bf 95.52$\pm$0.05} & {\bf 90.92$\pm$0.09} & {\bf 86.38$\pm$0.33} \\
\bottomrule
\end{tabular}
}
\caption{Performance (Accuracy$\pm$Std) of SGC and SGC+WGTL
under nettack.\label{net_sgc}}%Random Attack
% \vspace*{-1.5em}
\end{table}
\subsection{Performance of WGTL on other poisoning-type topological attacks: PGD and Meta-PGD}
From tables~\ref{pgd_results} and~\ref{mettpgd_results}, we observe that: (i) GCN+WGTL consistently outperforms GCN on these attacks which proves the effectiveness of our proposed WGTL,  and
% (ii) WGTL is the most effective against mettack and the least effective against Meta-PGD. For instance, against mettack on Polblogs (25\% perturbation rate), WGTL improves the accuracy by 
%  20\%, while against Meta-PGD, the accuracy gain is only 
% 1.4\%
% . The main reasons are two-fold: (i) Meta-PGD is a stronger attack than the mettack~\citep{metapgd}, and (ii) global poisoning attacks target graph topology and are supposed to be more challenging for the proposed {\it topology-based} defense WGTL. Lastly, 
(II) Meta-PGD is generally stronger attack than PGD for both GCN as well as GCN+WGTL. 
% From tables~\ref{pgd_results} and~\ref{mettpgd_results}, we observe that: (i) GCN+WGTL consistently outperforms GCN on these attacks which proves the effectiveness of our proposed WGTL,  (ii) WGTL is the most effective against mettack and the least effective against Meta-PGD. For instance, against mettack on Polblogs (25\% perturbation rate), WGTL improves the accuracy by 20\%, while against Meta-PGD, the accuracy gain is only 1.4\%. The main reasons are two-fold: (i) Meta-PGD is a stronger attack than the mettack~\citep{metapgd}, and (ii) global poisoning attacks target graph topology and are supposed to be more challenging for the proposed {\it topology-based} defense WGTL. Lastly, (III) Meta-PGD is generally stronger attack than PGD for both GCN as well as GCN+WGTL. 
%%%
\begin{table}[!h]
\small
\centering
\setlength\tabcolsep{1pt}
\resizebox{\columnwidth}{!}{
\begin{tabular}{llcccccc}
\toprule
\multirow{2}{*}{\textbf{Dataset}}&\multirow{2}{*}{\textbf{Model}}& \multicolumn{3}{c}{\textbf{Perturbation Rate}}
\\
\cmidrule(lr){3-5}%\cmidrule(lr){8-13}
           & & 0\% & 5\% & 10\% \\
\midrule
\multirow{2}{*}{{Cora-ML}} & GCN & 82.87$\pm$0.93 & 82.45$\pm$0.92 & 77.33$\pm$0.27 \\
& GCN + WGTL (ours) & {\bf 83.83$\pm$0.55} & {\bf 83.30$\pm$0.99} & {\bf 80.00$\pm$1.15} \\
\hline
\multirow{2}{*}{{Citeseer}}& GCN & 71.56$\pm$0.63 & 69.58$\pm$0.48 & 65.56$\pm$0.45 \\
&GCN + WGTL (ours) & {\bf 72.56$\pm$0.82} & {\bf 72.33$\pm$0.33} & {\bf 71.00$\pm$1.59} \\
\hline
\multirow{2}{*}{{Pubmed}}& GCN & 81.70$\pm$0.30 & 81.64$\pm$0.18 & 81.01$\pm$0.21 \\
& GCN + WGTL (ours) & {\bf83.93$\pm$0.06
} & {\bf 82.14$\pm$0.12} & {\bf 81.74$\pm$0.25} \\
\hline 
\multirow{2}{*}{{Polblogs}}& GCN & 94.40$\pm$1.48 & 91.17$\pm$2.27 & {89.92$\pm$1.43} \\
& GCN + WGTL (ours) & {\bf 95.95$\pm$0.15} & {\bf 91.45$\pm$0.51} & {\bf 90.02$\pm$1.16} \\ 
%\multirow{2}{*}{{Pubmed}}& GCN & 81.70$\pm$0.30 & 73.60$\pm$4.65 & 72.62$\pm$4.46 & 72.16$\pm$4.69 & 71.49$\pm$4.47 &  69.72$\pm$4.85\\
\bottomrule
\end{tabular}}
\caption{Performance (Accuracy$\pm$Std) %of GCN and GCN+WGTL 
under PGD-attack.\label{pgd_results}}%Random Attack
% \vspace*{-1em}
\end{table}
\begin{table}[t!]
\small
\centering
\setlength\tabcolsep{1pt}
\resizebox{\columnwidth}{!}{
\begin{tabular}{llcccccc}
\toprule
\multirow{2}{*}{\textbf{Dataset}}&\multirow{2}{*}{\textbf{Model}}& \multicolumn{3}{c}{\textbf{Perturbation Rate}}
\\
\cmidrule(lr){3-5}%\cmidrule(lr){8-13}%76.38$\pm$1.40
           & & 0\% & 5\% & 10\% \\
\midrule
\multirow{2}{*}{{Cora-ML}} & GCN & 82.87$\pm$0.93 & 79.30$\pm$0.86 & 76.26$\pm$0.92  \\
& GCN + WGTL (ours) & {\bf 83.83$\pm$0.55} & {\bf 79.57$\pm$1.10} & {\bf 76.52$\pm$0.81} \\
\hline
\multirow{2}{*}{{Citeseer}}& GCN & 71.56$\pm$0.63 & 67.89$\pm$0.59 & 66.80$\pm$0.79  \\
&GCN + WGTL (ours) & {\bf 72.56$\pm$0.82} & {\bf 69.38$\pm$0.27} & {\bf 67.57$\pm$0.67} \\
\hline
\multirow{2}{*}{{Pubmed}}& GCN & 81.70$\pm$0.30 & 77.24$\pm$0.14 & 73.56$\pm$0.17 \\
&GCN + WGTL (ours) & {\bf 83.93$\pm$0.06} & {\bf 78.97$\pm$0.20} & {\bf 75.22$\pm$0.16} \\
\hline
\multirow{2}{*}{{Polblogs}}& GCN & 94.40$\pm$1.48 & 83.46$\pm$2.13 & 78.08$\pm$0.73 \\
&GCN + WGTL (ours) & {\bf 95.95$\pm$0.15} & {\bf 85.52$\pm$0.70} & {\bf 81.28$\pm$0.31} \\
%\multirow{2}{*}{{Pubmed}}& GCN & 81.70$\pm$0.30 & 70.25$\pm$4.30 & 66.53$\pm$4.00 & 62.90$\pm$3.84 & 59.33$\pm$3.99 & 56.04$\pm$3.85 \\
\bottomrule
\end{tabular}}
% \vspace*{-1.5em}
\caption{Performance (Accuracy$\pm$Std) %of GCN and GCN+WGTL 
under Meta-PGD attack.\label{mettpgd_results}}%Random Attack
\end{table}
\subsection{Performance of WGTL on Heterophilic Graphs}
In the previous experiments, we have used four homophilic graph datasets: Cora-ML, Citeseer, Pubmed, and Polblogs. In this section, we aim to test the performance of WGTL and its poisoned variant WGTL$_\mathrm{P}$ on a heterophilic graph. Adopting the same attack configurations described in Appendix ~\ref{app:attacks}, we generate different perturbations (perturbation rates $0\%$ to $25\%$) of snap-patents graph~\citep{snap}. 

\textbf{Heterophilic Graph Dataset.} The snap-patents is a utility patent citation network. Node labels reflect the time the patent was granted, and the features are derived from the patent’s metadata. Following~\cite{kdd2022}, for better computational tractability, we sample a subset of the snap-patents data using a snowball sampling approach,
where a random 20\% of the neighbors for each traversed node are kept. 
% The snap-patents graph contains 4562 nodes, 12103 edges, 5 classes, and 269-dimensional node features. 
The homophily ratio for this dataset is as low as $0.134$. Hence, \cite{kdd2022} used this dataset as a benchmark heterophilic graph to test the robustness of many GNN architectures, including those proposed for heterophilic graphs.
% ~\citep{zhou2020graph,citemore}.

\textbf{Backbone GNN: H$_2$GCN.} Recently, it has been shown that GCN and other classical GNNs (e.g. GAT) perform poorly on heterophilic graphs~\citep{h2gcn2020,cpgnn}. More recently,~\cite{kdd2022} showed that these classical methods provide poor defence against adversarial attacks on heterophilic graphs as well. Thus, instead of GCN, we adopt  H${_2}$GCN~\citep{h2gcn2020} as the backbone architecture in this experiment.
H${_2}$GCN is proposed and popularly deployed to conduct classification on the heterophilic graphs.
H${_2}$GCN proposes a set of key design techniques to improve performance of GNNs on heterophilic graphs: (1) separation of ego- and
neighbor-embedding, (2) incorporation of higher-order neighborhoods, and (3) combination of
intermediate representations using skip-connections.

\begin{table}[!h]
\small
\centering
\resizebox{\columnwidth}{!}{
\begin{tabular}{@{}ccccc@{}}
\toprule
\textbf{Dataset} & \textbf{Models} & \multicolumn{3}{c}{\textbf{Perturbation Rate}} \\ \midrule
 &  & 0\% & 5\% & 10\% \\ \midrule
\multirow{4}{*}{\begin{tabular}[c]{@{}c@{}}snap-patents\end{tabular}} 
& GCN & 26.46$\pm$0.62 & 25.86$\pm$1.08 & 25.59$\pm$0.53 \\
% & GCN+WGTL & 25.70$\pm$1.24 & 25.86$\pm$0.48 & 26.21$\pm$1.13 \\
% \cmidrule(lr){2-5}
 & H\textsubscript{2}GCN & 27.71$\pm$0.86 & 27.55$\pm$0.19 & 28.62$\pm$0.38 \\
 & H\textsubscript{2}GCN+WGTL (ours) & \bf{27.72$\pm$0.85} & \bf{28.66$\pm$1.68} & \bf{28.79$\pm$1.00} \\ 
& H\textsubscript{2}GCN+WGTL$_{\mathrm{P}}$ (ours) & 27.72$\pm$0.85 & 28.02$\pm$0.38 & 28.72$\pm$0.76 \\ 
%  \cmidrule(lr){2-8}
% & GCN & 26.46$\pm$0.62 & 25.86$\pm$1.08 & 25.59$\pm$0.53 & 26.03$\pm$0.72 & 25.57 $\pm$1.11 & 25.74$\pm$0.78 \\
% & GCN+WGTL (ours) & 25.70$\pm$1.24 & 25.86$\pm$0.48 & 26.21$\pm$1.13 & 25.85$\pm$0.57 & 26.14$\pm$0.17 & 25.43$\pm$0.34 \\
\bottomrule
\end{tabular}%
}
\caption{Performance of WGTL on a heterophilic graph: snap-patents under Mettack.}
\label{tab:heterophi}
\end{table}
\textbf{Results and Observations.} For each perturbation rate, we run five experiments with H${_2}$GCN and H${_2}$GCN+WGTL on the corresponding perturbed snap-patents dataset, and report the mean$\pm$ standard deviation of the final classification accuracy in Table~\ref{tab:heterophi}. The results show that H${_2}$GCN+WGTL robustly improves the accuracy over H${_2}$GCN by up to $4\%$ across the perturbation rates. Note that the best-performing method APPNP~\citep{appnp}) on this dataset has been shown to have an accuracy of $27.76\%$ under $20\%$ perturbation (see Table~3 of~\citet{kdd2022}). Improving on that, we observe that H${_2}$GCN+WGTL achieves $28.21\%$ average accuracy under $20\%$ perturbation. Exploring other heterophilic graphs, and ways of further improving the performance of WGTL on heterophilic graphs are left as a future work. 
\begin{table}[!h]
\setlength\tabcolsep{2pt}
% \smaller 
\centering
\begin{tabular}{ccc}
\toprule
\textbf{Models} & \multicolumn{2}{c}{\textbf{Perturbation Rate}}  \\ 
\midrule
   & 0\%	& 10\% \\
\midrule
GCN	& 81.68 $\pm$ 0.73 & 71.16 $\pm$ 1.61 \\
GCN+WGTL$_{\mathrm{P}}$ (ours) &\textbf{82.62 $\pm$ 0.77} & \textbf{72.75 $\pm$ 1.24} \\
\bottomrule
\end{tabular}
\caption{Performance of WGTL on Cora-ML under Node-feature attack using Mettack.}
\label{tab:feature}
\end{table}

\subsection{Performance of WGTL under attacks on node-features}
We demonstrate a way to use WGTL to handle attacks on node features. We propose to feed the $k$-nearest neighbor graph ($k$-NNG) constructed from node features rather than the input graph in Figure~\ref{fig:wgtl}. Traditionally, the $k$-NNG is a graph where two nodes $u$ and $v$ are connected by an edge if the geodesic distance $d_{\graph}(u,v)$ is among the $k$-th smallest distances from $u$ to all other nodes. To handle node-feature perturbation, instead of using $d_{\graph}(u,v)$, we propose to use the cosine distance between feature vectors $X_u$ and $X_v$, defined as $\left(1-\frac{X_u \cdot X_v}{\|{X_u}\|_2 \|{X_v}\|_2}\right)$. 

Table~\ref{tab:feature} presents our experimental results on the Cora-ML graph under Mettack where the cosine distance-based $k$NN graph has been adopted. This approach is generally effective, as the observation indicates that the GCN+WGTL$_{\mathrm{P}}$ always outperforms GCN across clean and perturbed graphs.

\subsection{Performance of WGTL under adaptive adversarial attacks}
Recently~\citet{metapgd} proposed to use custom
adaptive attacks as a gold standard for evaluating defence methods. Hence, instead of using linearised GCN as a surrogate, we used the proposed GCN+WGTL as a surrogate to stress-test the ability of the proposed method to defend against a more powerful attack. 

\textbf{Adaptive Attack Settings \& Baselines.} The defense model WGTL, along with its hyper-parameters, is assumed to be a white box to the attacker, including the topological features on the clean graph. We do not assume that the attacker has access to the topological features on the perturbed graph because the PH topological feature computation is not differentiable; hence, it is not possible for the attacker to back-propagate gradients and compute the derivative of the adjacency matrix with respect to the features. The adaptive attack used was dubbed \emph{Aux-attack} in the codebase of ~\citet{metapgd}, where they used PGD attack~\cite{pgdattack} adaptively with tanh logit margin as the attacker's loss function.

We evaluate the effectiveness of WGTL against Aux-attack in both evasion and poisoning settings using the dataset splits from~\citet{gadc}. In the evasion attack setting, we train GCN+WGTL with clean topological features, perform Aux-attack to compute the perturbed graph, and finally, during evaluation, we plug in the topological features computed on the perturbed graph to do inference on test nodes. In the poisoning attack setting, we train GCN+WGTL (with the same architecture as the surrogate) except with the topological features computed on the poisoned graph. Hence, the target model is effectively GCN+WGTL$_\mathrm{P}$ in the poisoning attack setting. Finally, we evaluate the trained GCN+WGTL$_\mathrm{P}$ on the test nodes. 

Apart from the GCN baseline, we have also tested more recently proposed two baselines in these settings: Graph adversarial diffusion convolution (GADC)~\cite{gadc} and GNNGuard~\cite{zitnikGNNGuard}.

\textbf{Results and Observations.} The results from one run on the Cora-ML and Polblogs dataset are presented in Table~\ref{tab:adaptive}. The results demonstrate that our method significantly outperforms the baselines on both evasion and poisoning adaptive attack settings. For instance, it achieves up to 5.7\% improvement over GADC(IV) on Cora-ML and up to 69\% improvement on polblogs. 
% improvements=>
% Cora-ML: 
% --- wrt GCN: 0.4%, 5.0% 
% --- wrt GNNGuard: 2.8%, 4.5%
% --- wrt GADC(IV): 5.7%, 3.4%
% Polblogs
% -- wrt GCN: 1.8%, 20.2%
% -- wrt GNNGuard: 1.4%, 14%
% -- wrt GADC(IV): 69%, 5.3%

\begin{table}[t]
\setlength\tabcolsep{2pt}
\resizebox{0.98\columnwidth}{!}{
\begin{tabular}{@{}cccc@{}}
\toprule
\textbf{Dataset} & \textbf{Model} & \textbf{Evasion Attack} & \textbf{Poisoning Attack} \\ \midrule
\multirow{6}{*}{Cora-ML} & GCN & 60.06 & 60.46 \\
 & GCN+WGTL$_\mathrm{P}$ (ours) & \textbf{60.31} & \textbf{63.48} \\ \cmidrule{2-4}
 & GNNGuard & 61.92 & 60.81 \\
 & GNNGuard+WGTL$_\mathrm{P}$ (ours) & \textbf{63.63} & \textbf{63.53} \\ \cmidrule{2-4}
  & GADC(IV) & 72.08 & 73.14 \\
 & GADC(IV)+WGTL$_\mathrm{P}$ (ours) & \textbf{*76.16} & \textbf{75.65} \\ 
 \midrule 
\multirow{6}{*}{Polblogs} & GCN & 52.35 & 53.17 \\
 & GCN+WGTL$_\mathrm{P}$ (ours) & \textbf{53.27} & \textbf{63.91}\\ 
 \cmidrule{2-4}
 & GNNGuard & 49.69 & 47.03\\
 & GNNGuard+WGTL$_\mathrm{P}$ (ours) & \textbf{50.41} & \textbf{53.58} \\ \cmidrule{2-4}
 & GADC(IV) & 52.04 & 52.04 \\
 & GADC(IV)+WGTL$_\mathrm{P}$ (ours) & \textbf{*87.93} & \textbf{54.81} \\ 
 \bottomrule
\end{tabular}
}
\caption{Performance from one run under PGD-based adaptive adversarial attack with 20\% perturbation rate.}
\label{tab:adaptive}
\end{table}

\subsection{Performance of WGTL with persistence curve vectorization} 
Persistence curve~\cite{chung2022persistence} is a family of representations that are functions of the multiset of birth and death pairs $\{(b_i,d_i)\in \mathbb{R}^2: d_i > b_i, \forall i\in \{1,2,\ldots,n\}\}$. Among the various functions proposed by~\citet{chung2022persistence}, the \emph{Normalized life curve}  has been recommended by the authors for its stability, computational efficiency, and performance. The normalized life curve is defined as the following transformation:
\[
\frac{d_i - b_i}{\sum_{i=1}^{n} (d_i-b_i)}
\]
We have replaced the persistence image vectorization in our local topology encoding (Component I) and global topology encoding (Component II) with the normalized life curve vectorization and present the accuracy on the poisoned Cora-ML and Polblogs in Table~\ref{tab:results_pc}. We observe that on certain datasets such as Cora-ML, the persistence curve can provide slightly better robustness than the Persistence Image representations. However, the persistence curves produce a higher standard deviation than the persistence image. Since both vectorization methods produce better accuracy than the baseline, the proposed method is agnostic to the choice of alternative vectorizations of persistent topological features. 

\begin{table}[t]
\small
\centering
\resizebox{\columnwidth}{!}{
\begin{tabular}{@{}ccccc@{}}
\toprule
\textbf{Dataset} & \textbf{Models} & \multicolumn{3}{c}{\textbf{Perturbation Rate}} \\ \midrule
 &  & 0\% & 5\% & 10\% \\ \midrule
\multirow{3}{*}{\begin{tabular}[c]{@{}c@{}}Cora-ML\end{tabular}} 
 & GCN & 82.87$\pm$0.83 & 76.55$\pm$0.79 & 70.39$\pm$1.28 \\
 & GCN+WGTL$_{\mathrm{P}}$(PI) & \textbf{83.83$\pm$0.55} & 76.96$\pm$0.76 & 71.31$\pm$0.85 \\ 
& GCN+WGTL$_{\mathrm{P}}$(PC) & 83.72$\pm$0.50 & \textbf{78.16$\pm$1.04 }& \textbf{72.54$\pm$1.48} \\ 
 \midrule
\multirow{3}{*}{\begin{tabular}[c]{@{}c@{}}Polblogs\end{tabular}} 
  & GCN & 94.40$\pm$1.47 & 71.41$\pm$2.42 & 69.16$\pm$1.86 \\
 & GCN+WGTL$_{\mathrm{P}}$(PI) & \bf{95.95$\pm$0.15} & 73.02$\pm$1.13 & \bf{74.52$\pm$0.28} \\ 
& GCN+WGTL$_{\mathrm{P}}$(PC) & 95.19$\pm$0.32 & \textbf{73.05$\pm$1.16 }& 69.97$\pm$1.14 \\ 
\bottomrule
\end{tabular}%
}
\caption{Performance of WGTL with different topological vectorizations on Cora-ML under Mettack. PC = Normalized life curve, PI = Persistence image}
\label{tab:results_pc}
\end{table}
% Cora=>
%  GADC => nhid=16, 
% GADC +WGTL => nhid=128, lambda = 0.01, patience = 50, aggregation_method='weighted_sum'
% Cora =>
% Evasion test acc:  0.7615693807601929
% Poisoned test acc: 0.7258551120758057
% Citeseer=>
% GADC => 71.98, 70.44, 68.19
% GADC + WGTL => 7174, 61.79, 63.98
% Polblogs=>
% 52.04, 52.04, 52.04
% 94.98, 87.93, 49.38

\end{document}